\documentclass[nohyperref]{article}

\usepackage{microtype}
\usepackage{graphicx}
\usepackage{subfigure}
\usepackage{booktabs} 

\usepackage{hyperref}

\usepackage[accepted]{icml2022}

\usepackage{amsmath}
\usepackage{amssymb}
\usepackage{mathtools}
\usepackage{amsthm}

\usepackage{booktabs}
\usepackage{bm}
\usepackage{url}  
\usepackage{mathrsfs}
\usepackage{multirow}
\usepackage{balance}
\usepackage{amstext}
\usepackage{overpic}
\usepackage{color}
\usepackage{enumitem}
\usepackage{fancyvrb}
\usepackage{graphicx}
\usepackage{caption}
\usepackage{multirow}
\usepackage{bm}
\usepackage{algorithm, algorithmic}
\usepackage{subfigure}

\usepackage[capitalize,noabbrev]{cleveref}

\theoremstyle{plain}
\newtheorem{theorem}{Theorem}

\newtheorem{corollary}{Corollary}
\theoremstyle{definition}

\theoremstyle{remark}
\newtheorem{remark}{Remark}

\newtheorem{define}{Definition}
\definecolor{gray}{RGB}{150, 150, 150}
\definecolor{orange}{RGB}{255, 97, 0}

\def\Algname{\underline{SH}apley \underline{E}xplanation \underline{A}ccele\underline{R}ation}
\def\Algnameb{Shapley Explanation Acceleration}
\def\Algnameabbr{SHEAR}

\usepackage[textsize=tiny]{todonotes}

\icmltitlerunning{Accelerating Shapley Explanation via Contributive Cooperator Selection}

\begin{document}

\twocolumn[
\icmltitle{Accelerating Shapley Explanation via Contributive Cooperator Selection}



\icmlsetsymbol{equal}{*}


\begin{icmlauthorlist}
\icmlauthor{Guanchu Wang}{equal,aff1}
\icmlauthor{Yu-Neng Chuang}{equal,aff1}
\icmlauthor{Mengnan Du}{aff2}
\icmlauthor{Fan Yang}{aff1} \\
\icmlauthor{Quan Zhou}{aff3}
\icmlauthor{Pushkar Tripathi}{aff3}
\icmlauthor{Xuanting Cai}{aff3}
\icmlauthor{Xia Hu}{aff1}
\end{icmlauthorlist}

\icmlaffiliation{aff1}{Department of Computer Science, Rice University}
\icmlaffiliation{aff2}{Department of Computer Science and Engineering, Texas A\&M University}
\icmlaffiliation{aff3}{Meta Platforms, Inc.}

\icmlcorrespondingauthor{Guanchu Wang}{guanchu.wang@rice.edu}
\icmlcorrespondingauthor{Xia Hu}{xia.hu@rice.edu}

\icmlkeywords{Machine Learning, ICML}

\vskip 0.3in
]



\printAffiliationsAndNotice{\icmlEqualContribution} 

\begin{abstract}

Even though Shapley value provides an effective explanation for a DNN model prediction, the computation relies on the enumeration of all possible input feature coalitions, which leads to the exponentially growing complexity. To address this problem, we propose a novel method \Algnameabbr{} to significantly accelerate the Shapley explanation for DNN models, where only a few coalitions of input features are involved in the computation. The selection of the feature coalitions follows our proposed Shapley chain rule to minimize the absolute error from the ground-truth Shapley values, such that the computation can be both efficient and accurate. To demonstrate the effectiveness, we comprehensively evaluate \Algnameabbr{} across multiple metrics including the absolute error from the ground-truth Shapley value, the faithfulness of the explanations, and running speed. The experimental results indicate \Algnameabbr{} consistently outperforms state-of-the-art baseline methods across different evaluation metrics, which demonstrates its potentials in real-world applications where the computational resource is limited.
The source code is available at \url{https://github.com/guanchuwang/SHEAR}.

\end{abstract}




\section{Introduction}


Despite the remarkable achievement of deep neural networks (DNNs) in a variety of fields, the black-box nature of DNNs still limits its deployment in domains where model explanations are required for acquiring trustful results, such as healthcare~\cite{esteva2019guide}, finance~\cite{caruana2020intelligible} and recommender systems~\cite{yang2018towards}. 
Explaining the behavior of DNNs is a significant problem due to both the practical requirements of the stakeholders as well as the regulations in different domains, e.g., GDPR~\cite{goodman2017european, floridi2019establishing}.
To overcome the black-box nature of DNNs, existing work has developed various techniques for model interpretation such as gradient-based methods~\cite{sundararajan2017axiomatic}, casual interpretation~\cite{luo2020causal}, counterfactual explanation~\cite{yang2021model} and Shapley value explanation~\cite{lundberg2017unified}.
Among these, the Shapley value~\cite{shapley201617} has emerged as a popular explanation approach due to its strong theoretical properties~\cite{chuang2023efficient}.



The Shapley value provides a natural and effective explanation for DNNs from the perspective of cooperative game theory~\cite{kuhn1953contributions, winter2002shapley, roth1988introduction}.
The explanation can be adaptive to individual features, an instance, or the representation of global feature contribution~\cite{covert2020understanding}.
However, the Shapley explanation is known to be an NP-hard problem with extremely high computational complexity, which prevents its application to real-world scenarios.
The brute-force algorithm to calculate the exact Shapley values requires the enumeration of all possible input feature coalitions, where the complexity grows exponentially with the feature number~\cite{van2021tractability}.
Hence, it is crucial to reduce the computational complexity of Shapley explanation.
Existing work can be categorized into two groups to overcome this challenge.
The first group studies specific approximation of Shapley value for DNN models~\cite{chen2018shapley, ancona2019explaining, jia2019towards, wang2021shapley}.
Even though these kinds of work contribute to efficient explanations for the prediction of DNN models, they suffer from the inevitable gap with the ground-truth Shapley explanation due to the approximation~\cite{liu2021synthetic}, and lack of flexibility to deal with different types of models.
Another group proposes the regression of Shapley values based on model evaluation on either the sampling of feature coalitions or permutations~\cite{kokhlikyan2020captum, covert2020understanding}.
Although these kinds of methods provide the effective Shapley explanation for DNN models, they require large numbers of model evaluations, which is inefficient and often performs an undesirable trade-off between the computational complexity and interpretation performance~\cite{covert2021improving, ribeiro2016should, liu2021synthetic}. 

In this work, we propose a novel method to reduce the complexity of Shapley value estimation for the acceleration of DNN explanation.
Instead of using all possible input feature coalitions, we focus on selecting a few features to generate the coalitions for the estimation such that the complexity can be significantly reduced.
Specifically, we first propose the Shapley chain rule to indicate that the absolute estimation error is related to the selection of input features, and those can optimize the estimation are termed as \emph{contributive cooperators}.
Then, we propose \Algname~(\Algnameabbr{}) following the Shapley chain rule to conduct contributive cooperator selection in the estimation of Shapley value.
In this way, the enumeration of all possible feature coalitions can be avoided and the explanation achieves significant acceleration.
To demonstrate \Algnameabbr{} enables both accurate and efficient Shapley explanation, we conduct experiments on three benchmark datasets to compare it with five state-of-the-art baseline methods using five evaluation metrics.
The evaluation involves two metrics to compare the explanation with ground-truth Shapley values: the absolute estimation error and accuracy of feature importance ranking~\cite{wojtas2020feature};
two metrics to evaluate the explanation via model perturbation: the faithfulness~\cite{liu2021synthetic} and monotonicity~\cite{luss2019generating};
and algorithmic throughput~\cite{teich2018plaster} to evaluate the running speed.
The contributions of this work are summarized as follows:
\begin{itemize}[leftmargin=10pt, topsep=0pt]
\setlength{\parskip}{1pt}
\setlength{\parsep}{0pt}
\setlength{\itemsep}{0pt}

    \item We propose the Shapley chain rule to theoretically guide the complexity reduction of Shapley value estimation.
    
    \item Following the Shapley chain rule, we further propose \Algnameabbr{} for Shapley explanation acceleration.
    
    \item Experimental results over three datasets indicate that our SHEAR works more efficiently than state-of-the-art methods without degradation of interpretation performance.
    
    
    
\end{itemize}

\section{Preliminaries}

In this section, we introduce the notations used throughout this work and provide an overview of Shapley values.

\subsection{Notations}

We consider an arbitrary DNN model $f$ and input feature $\boldsymbol{x} \in \mathcal{X}$, where $x_1, \cdots, x_M$ denotes the value of input feature $1, \cdots, M$, respectively, and each feature has either continuous or categorical value.
To formalize the contribution of each feature to the prediction, the inference of $f$ is regarded as a cooperative game on the feature set $\mathcal{U} = \{ 1, \cdots, M \}$, where we let $f_v: 2^M \to \mathbb{R}$ denote the value function.
Specifically, for a feature coalition~(i.e., subset) $\boldsymbol{S} \subseteq \mathcal{U}$, $f_v (\boldsymbol{S})$ returns the prediction based on the features in coalition $\boldsymbol{S}$ that are marginalized over the features not in coalition $\boldsymbol{S}$, which is given as follows
\begin{equation}
\setlength\abovedisplayskip{1mm}
\setlength\belowdisplayskip{1mm}
\label{eq:value_func}
    f_v(\boldsymbol{S}) = \mathbb{E} \big[ f( x_1, \cdots, x_M) \mid \boldsymbol{x}_{\mathcal{U} \setminus \boldsymbol{S}} \sim p(\boldsymbol{x}_{\mathcal{U} \setminus \boldsymbol{S}}) \big],
\end{equation}
where $\boldsymbol{x}_{\mathcal{U} \setminus \boldsymbol{S}}$ denotes $[x_j \mid j \in \mathcal{U} \setminus \boldsymbol{S}]$, and $p(\boldsymbol{x}_{\mathcal{U} \setminus \boldsymbol{S}})$ denotes the joint distribution of $x_j$ for $j \in \mathcal{U} \setminus \boldsymbol{S}$.

However, the marginalized value function is difficult to estimate following Equation~(\ref{eq:value_func}), since it depends on the enumeration of data instances over the whole dataset.
A widely used solution~\cite{lundberg2017unified, wang2021shapley, kokhlikyan2020captum, covert2021improving} is to approximate Equation~(\ref{eq:value_func}) as follows
\begin{equation}
\setlength\abovedisplayskip{1mm}
\setlength\belowdisplayskip{1mm}
    \label{eq:value_func_approx}
    f_v(\boldsymbol{S}) \approx f \big( \boldsymbol{x}_{\boldsymbol{S}}, \overline{\boldsymbol{x}}_{\mathcal{U} \setminus \boldsymbol{S}} \big),
\end{equation}
where $\boldsymbol{x}_{\boldsymbol{S}}$ denotes $[x_j \mid j \in \boldsymbol{S}]$; and $\overline{\boldsymbol{x}}_{\mathcal{U} \setminus \boldsymbol{S}} = \mathbb{E}[ \boldsymbol{x}_{\mathcal{U} \setminus \boldsymbol{S}} \mid \boldsymbol{x}_{\mathcal{U} \setminus \boldsymbol{S}} \sim p(\boldsymbol{x}_{\mathcal{U} \setminus \boldsymbol{S}})]$ denotes the reference values\footnote{\scriptsize Other statistic value can also be adopted for the reference value.} of feature $j \in \mathcal{U} \setminus \boldsymbol{S}$.
We follow this approximation in this work.




\subsection{Shapley Value}

Shapley value regards the input features to the DNN model as the \emph{cooperators}, and estimates the feature contribution from the perspective of cooperative game theory~\cite{kuhn1953contributions}.
Specifically, it assigns an importance value $\phi_i(f_v, \mathcal{U})$ to indicate the contribution of feature~$i \in \mathcal{U}$ to the DNN prediction, and formalizes the feature contribution following the brute-force algorithm in Equation~(\ref{eq:shapley_value}).
According to Equation~(\ref{eq:shapley_value}), Shapley value adopts the preceding difference $f_v(\{ i \} \cup \boldsymbol{S}) - f_v(\boldsymbol{S})$ to indicate the contribution of feature~$i$ considering the features in coalition $\boldsymbol{S}$, and enumerates the feature coalitions throughout the cooperators, which are the $M$ input features.
The average preceding difference considering all possible feature coalitions indicates the contribution of feature~$i$,
\begin{equation}
\setlength\abovedisplayskip{1mm}
\setlength\belowdisplayskip{1mm}
\label{eq:shapley_value}
    \phi_i (f_v, \mathcal{U}) \!=\! \frac{1}{M} \!\!\!\! \sum_{\boldsymbol{S} \subseteq \mathcal{U} \setminus \{ i \}} \!\!\!\! \binom{M \!-\! 1}{|\boldsymbol{S}|}^{-1} \!\!\!\!\! \big[ f_v(\{ i \} \cup \boldsymbol{S}) \!-\! f_v(\boldsymbol{S}) \big].
\end{equation}

The brute-force algorithm takes all of the $M$ features as the cooperators, and relies on $2^M$ times of model evaluation to estimate the contribution of feature $i$, which has the computational complexity exponentially growing with the feature number $T [ \phi_i (f, \mathcal{U}) ] = O(2^M)$.
In this work, we propose a low-complexity estimation of the feature contribution $\{ \hat{\phi}_i \}_{1 \leq i \leq M}$, where the contribution of each feature $\hat{\phi}_i$ is estimated based on $N$ times of model evaluation, and we have $N \ll 2^M$.
The estimation aims to minimize the following absolute error taking the ground-truth Shapley value~(GT-Shapley value) as the reference,
\begin{equation}
\setlength\abovedisplayskip{0mm}
\setlength\belowdisplayskip{0mm}
\label{eq:phi_hat_objective}
\min \sum_{i=1}^M |\phi_i(f_v, \mathcal{U}) - \hat{\phi}_i|,
\end{equation}
where $\hat{\phi}_i$ denotes the estimated contribution of feature $i$. 








\section{Shapley Chain Rule}
\label{sec:chain_rule}

In this section, we propose \emph{Shapley Chain Rule} in Theorem~\ref{theorem:shapley_chain_rule} to provide theoretical instructions for the Shapley explanation acceleration.
Note that the brute-force algorithm considers all of the $M$ input features as the cooperators, which leads the complexity of $\phi_i(f_v, \mathcal{U})$ to grow exponentially with the feature number.
Shapley chain rule provides the estimation of feature contribution based on only a few features as the cooperators to significantly reduce the computational complexity.
We give the proof of Theorems~\ref{theorem:shapley_chain_rule} and~\ref{prop:error_term_bound} in Appendix~\ref{sec:proof_theorem_1} and~\ref{sec:proof_theorem_2}, respectively.



\begin{theorem}[Shapley Chain Rule]
\label{theorem:shapley_chain_rule}
For any differentiable value function $f_v: 2^M \to \mathbb{R}$, the contribution of feature $i$ to $f_v(\mathcal{U})$ satisfies 
\begin{equation}
\setlength\abovedisplayskip{2mm}
\setlength\belowdisplayskip{2mm}
\phi_i( f_v, \mathcal{U} ) = \phi_i( f_v, \mathcal{U} \setminus \{j\} ) + \Delta_{i,j} + o_{i,j}, 
\nonumber
\end{equation}
where $j \in \mathcal{U} \setminus \{i\}$ denotes another feature; $\phi_i( f_v, \mathcal{U} \setminus \{j\} )$ denotes the contribution of feature $i$ to $f_v(\mathcal{U} \setminus \{j\})$; $o_{i,j} \!=\! o(x_i \!-\! \bar{x}_i) \!+\! o(x_j \!-\! \bar{x}_j)$; and the error term $\Delta_{i,j}$ is given by
\begin{equation}
\begin{aligned}\\[-6mm]
\Delta_{i,j} = &(x_i - \bar{x}_i) (x_j - \bar{x}_j) 
\\
&\!\!\!\!\! \sum_{\boldsymbol{S} \subseteq \mathcal{U} \setminus \{ i,j \}} \!\!\!\!\! \frac{ \nabla_{i,j}^2 f_v( \boldsymbol{S} \!\cup\! \{ i,j \}) \!+\! \nabla_{j,i}^2 f_v( \boldsymbol{S} \!\cup\! \{ i,j \} ) }{2 (M - |\boldsymbol{S}| - 1) \binom{M}{|\boldsymbol{S}|+1}},
\nonumber
\end{aligned}
\end{equation}\\[-4mm]
where $x_i$ denotes the value of feature $i$; $\bar{x}_i$ denotes the reference value of feature $i$; and $\nabla_{i,j}^2 f_v ( \boldsymbol{S} ) = \frac{\partial^2 f ( \boldsymbol{x}_{\boldsymbol{S}},  \bar{\boldsymbol{x}}_{\mathcal{U} \setminus \boldsymbol{S}} )}{\partial x_i \partial x_j}$ denotes the cross-gradient towards $x_i$ and $x_j$. 
\end{theorem}



\begin{remark}[Complexity Reduction]
\label{rk:complexity_reduce}
The computational complexity of $\phi_i( f_v, \mathcal{U} \setminus \{j\})$ equals to the half of $\phi_i( f_v, \mathcal{U})$.
\end{remark}


According to Remark~\ref{rk:complexity_reduce}, the computational complexity can be significantly reduced if we remove feature $j$ from the cooperators such that the contribution of feature $i$ can be estimated by $\phi_i( f_v, \mathcal{U} \setminus \{ j \} )$.
However, it causes a significant estimation error according to Theorem~\ref{theorem:shapley_chain_rule}.
To reduce the complexity but without loss of accuracy, we propose Theorem~\ref{prop:error_term_bound} to bound the error term.
\begin{theorem}[Upper Bound of Error Term]
\label{prop:error_term_bound}
For any features $i \neq j \in \mathcal{U}$, the upper bound of the absolute gap between $\phi_i( f_v, \mathcal{U} )$ and $\phi_i( f_v, \mathcal{U} \setminus \{ j \} )$ is given by
\begin{equation}
\setlength\abovedisplayskip{2mm}
\setlength\belowdisplayskip{2mm}
\label{eq:error_upper_bound}
\big| \phi_i( f_v, \mathcal{U} ) \!-\! \phi_i( f_v, \mathcal{U} \setminus \{ j \} )  \big| \!\leq\! \epsilon_{i,j} |x_i \!-\! \bar{x}_i| |x_j \!-\! \bar{x}_j|,
\end{equation}
where $\epsilon_{i,j}$ relies on the gradient towards $x_i$ and $x_j$ by
\begin{equation}
\setlength\abovedisplayskip{1mm}
\setlength\belowdisplayskip{2mm}
\label{eq:epsilon}
\epsilon_{i,j} \!=\! \max\limits_{\mathcal{V} \subseteq \mathcal{U} \setminus \{ i,j \}} \frac{1}{4} \big| \nabla_{i,j}^2 f_v (\mathcal{U} \setminus \mathcal{V}) \!+\! \nabla_{j,i}^2 f_v (\mathcal{U} \setminus \mathcal{V}) \big|.
\end{equation}
\end{theorem}


\begin{remark}
\label{rk:cumulative_error}
For any feature subset $\mathcal{S} \subseteq \mathcal{U} \setminus \{ i \}$, the absolute gap between $\phi_i( f_v, \mathcal{U} )$ and $\phi_i( f_v, \mathcal{S} \cup \{ i \} )$ is bounded by
\begin{equation}
\setlength\abovedisplayskip{2mm}
\setlength\belowdisplayskip{1mm}
\label{eq:upper_bound_error_group}
\! \big| \phi_i( f_v, \mathcal{U} ) \!-\! \phi_i( f_v, \mathcal{S} \!\cup\! \{ i \} ) \big| \!\leq\! \!\!\!\!\!\!\!\! \sum_{j \in \mathcal{U} \setminus \mathcal{S} \setminus \{ i \} } \!\!\!\!\!\!\!\! \epsilon_{i,j} |x_i \!-\! \bar{x}_i| |x_j \!-\! \bar{x}_j|. \!\!\!\!\!\!\!\!
\end{equation}

\end{remark}





Note that Equation~(\ref{eq:upper_bound_error_group}) provides the upper bound for the absolute estimation error in Equation~(\ref{eq:phi_hat_objective}).
We take $\hat{\phi}_i = \phi_i (f, \mathcal{S}_i \cup \{ i \})$, for $1 \leq i \leq M$.
In this way, the problem of reducing the estimation complexity can be formulated into selecting the optimal contributive cooperators for each feature to minimize the worst-case absolute estimation error.
According to the upper bound in Equation~(\ref{eq:upper_bound_error_group}), we have
\begin{align} \\[-6mm]
\mathcal{S}_i &= \mathop{\arg\min}_{\mathcal{S} \subseteq \mathcal{U} \setminus \{ i \}} | \phi_i (f_v, \mathcal{U}) - \phi_i (f_v, \mathcal{S} \cup \{ i \}) |,
\nonumber
\\
\label{eq:S_selection_0}
&\sim \mathop{\arg\min}_{\mathcal{S} \subseteq \mathcal{U} \setminus \{ i \}} \sum_{j \in \mathcal{U} \setminus \mathcal{S}} \epsilon_{i,j} |x_i - \bar{x}_i| |x_j - \bar{x}_j|,
\end{align} \\[-4mm]
where we have $|\mathcal{S}_i| = \log_2 \frac{N}{2}$ to constrain the number of contributive cooperators such that the estimation complexity satisfies $T[ \phi_i (f, \mathcal{S}_i \cup \{ i \}) ] = O(N)$, where $N$ denotes the time of model evaluation which depends on the available limited computational resource.


For the convenience of optimization, we transform the $\arg \min$ problem in Equation~(\ref{eq:S_selection_0}) to $\arg \max$ as the objective function of the contributive cooperator selection as follows,
\begin{equation}
\setlength\abovedisplayskip{1mm}
\setlength\belowdisplayskip{1mm}
\label{eq:S_selection}
\mathcal{S}_i = \mathop{\arg\max}_{\mathcal{S} \subseteq \mathcal{U} \setminus \{ i \} \atop |\mathcal{S}| = \log_2 (N/2)} \sum_{j\in \mathcal{S}} \epsilon_{i,j} |x_i - \bar{x}_i| |x_j - \bar{x}_j|.
\end{equation}

\begin{figure*}
\setlength{\abovecaptionskip}{1mm}
\setlength{\belowcaptionskip}{-3mm}
\centering
\includegraphics[width=0.98\linewidth]{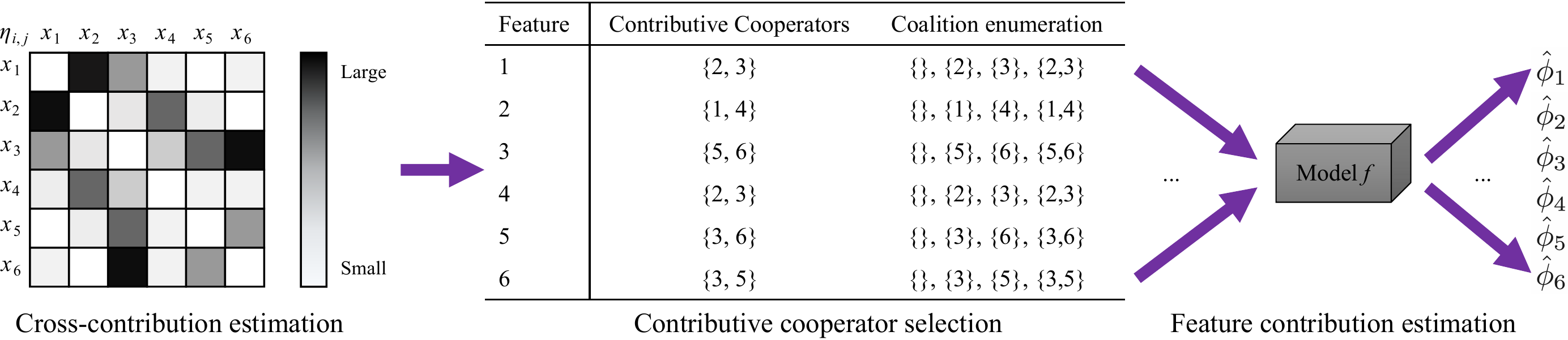}
    \caption{Algorithmic configuration of \text{\Algnameabbr{}}.}
    \label{fig:eff_shap}
\end{figure*}

\section{\Algnameb{}}
\label{sec:eff_shap}

In this section, we propose the  \Algname{}~(\Algnameabbr{}) to provide efficient Shapley explanations for DNN models.
First, we propose the definition of feature \emph{cross-contribution} for \Algnameabbr{} to approximately solve Equation~(\ref{eq:S_selection}).
Second, we adopt antithetical sampling for \Algnameabbr{} to promote the estimation.
Afterwards, we give the details of \Algnameabbr{}.
Finally, the empirical complexity analysis is provided for \Algnameabbr{}.

\subsection{Feature Cross-contribution}

The selection of contributive cooperators is challenging following Equation~(\ref{eq:S_selection}) due to the high computational complexity of $\epsilon_{i,j}$ following Equation~(\ref{eq:epsilon}).
To address this problem, we propose to approximate $\epsilon_{i,j}$ into $\hat{\epsilon}_{i,j} \approx \frac{1}{4} | \nabla^2_{i,j} f_v(\mathcal{U}) + \nabla^2_{j,i} f_v(\mathcal{U}) |$, and propose feature \emph{cross-contribution} in Definition~\ref{def:feature_coop} for reaching the approximately optimal solution of Equation~(\ref{eq:S_selection}) in \Algnameabbr{}.
\begin{define}[Cross-contribution]
\label{def:feature_coop}
For features $i \neq j \in \mathcal{U}$, the cross-contribution of features $i$ and $j$ is defined as  
\begin{equation}
\setlength\abovedisplayskip{1mm}
\setlength\belowdisplayskip{1mm}
\label{eq:cross_contribution}
    \eta_{i,j} = |x_i - \bar{x}_i| \Big| \nabla^2_{i,j} f_v(\mathcal{U}) + \nabla^2_{j,i} f_v(\mathcal{U}) \Big| |x_j - \bar{x}_j|,
\end{equation}
where $\nabla_{i,j}^2 f_v ( \mathcal{U} ) = \frac{\partial^2 f ( \boldsymbol{x} )}{\partial x_i \partial x_j}$ denotes the cross-gradient\footnote{\scriptsize torch.autograd provides APIs to estimate the backward gradient of DNNs.} towards features $i$ and $j$; and we have the cross-gradient satisfying $\nabla_{i,j}^2 f_v ( \mathcal{U} ) = \nabla_{j,i}^2 f_v ( \mathcal{U} )$ for DNNs.

\end{define}

\begin{remark}
Definition~\ref{def:feature_coop} considers a special case $\mathcal{V} \!=\! \varnothing$ for Equation~(\ref{eq:epsilon}) to simplify the computation.
\end{remark}


Intuitively, $\eta_{i,j}$ indicates the strength of cooperation between features $i$ and $j$; and the feature subset that can maximize the cross-contribution with feature $i$ provides an approximately optimal solution for Equation~(\ref{eq:S_selection}).
In such a manner, \Algnameabbr{} selects the contributive cooperators $\mathcal{S}_i$ for each feature $i \in \mathcal{U}$ following 
\begin{equation}
\setlength\abovedisplayskip{1mm}
\setlength\belowdisplayskip{1mm}
\label{eq:S_selection2}
    \mathcal{S}_i = \mathop{\arg\max}_{\mathcal{S} \subseteq \mathcal{U} \setminus \{ i \} \atop |\mathcal{S}| = \log_2 (N/2)} \sum_{j\in \mathcal{S}} \eta_{i,j},
\end{equation}
where we have the constraint $|\mathcal{S}| = \log_2 \frac{N}{2}$ for the $\arg \max$ problem such that $T(\hat{\phi}_i) \!=\! O(N)$;
and the optimal solution of Equation~(\ref{eq:S_selection2}) can be achieved via ranking and greedy search under $O(M \log M)$ complexity.
After this, \Algnameabbr{} estimates the contribution of feature~$i$ following
\begin{equation}
\setlength\abovedisplayskip{1mm}
\setlength\belowdisplayskip{1mm}
\label{eq:eff_shap_wo_as}
    \hat{\phi}_i = \frac{1}{|\mathcal{S}_i|\!+\!1\!} \!\! \sum_{\boldsymbol{S} \subseteq \mathcal{S}_i} \!\! \binom{|\mathcal{S}_i|}{|\boldsymbol{S}|}^{-1} \!\!\!\!\!\! \big[ f_v(\{ i \} \cup \boldsymbol{S}) - f_v(\boldsymbol{S}) \big]. 
\end{equation}

\subsection{Antithetical Sampling}



Note that Equation~(\ref{eq:eff_shap_wo_as}) totally ignores the influence of non-contributive features $j \! \in \! \mathcal{U} \! \setminus \! \mathcal{S}_i \! \setminus \! \{ i \}$ to the value function, which leads to sub-optimal solutions.
To address this problem but without extra calculation, \Algnameabbr{} employs the antithetical sampling~(AS)~\cite{lomeli2019antithetic} to fix the preceding difference in Equation~(\ref{eq:eff_shap_wo_as}) to promote the estimation.
Specifically, following the enumeration of $\boldsymbol{S} \subseteq \mathcal{S}_i$ in Equation~(\ref{eq:eff_shap_wo_as}) where $|\mathcal{S}_i| \!=\! \log_2 \frac{N}{2}$, let $\boldsymbol{S}_1, \cdots, \boldsymbol{S}_{N/2} \subseteq \mathcal{S}_i$ denote the entirely possible feature coalitions~(i.e. subsets) on $\mathcal{S}_i$ where $\boldsymbol{S}_1 \neq \cdots \neq \boldsymbol{S}_{N/2}$.
In each case of the enumeration where $\boldsymbol{S} \!=\! \boldsymbol{S}_n$, the AS revises the preceding difference $f_v(\{ i \} \! \cup \! \boldsymbol{S}_n) \!-\! f_v(\boldsymbol{S}_n)$ into $f_v(\{ i \} \!\cup\! \boldsymbol{S}_n \!\cup\! \boldsymbol{V}_n) \!-\! f_v(\boldsymbol{S}_n \!\cup\! \boldsymbol{V}_n)$ such that the contribution of feature $i$ is estimated as follows $\!\!\!\!$
\begin{equation}
\setlength\abovedisplayskip{1mm}
\setlength\belowdisplayskip{1mm}
\label{eq:eff_shap_feature_contribution}
    \! \hat{\phi}_i \!=\! \frac{1}{|\mathcal{S}_i|\!+\!1\!} \!\! \sum_{\boldsymbol{S}_n \subseteq \mathcal{S}_i} \!\!\!\! \binom{|\mathcal{S}_i|}{|\boldsymbol{S}|}^{-1} \!\!\!\!\!\!\!\! \big[ f_v \! ( \{i\} \!\cup\! \boldsymbol{S}_n \!\cup\! \boldsymbol{V}_n ) \!-\! f_v \! ( \boldsymbol{S}_n \!\cup\! \boldsymbol{V}_n ) \big]\!, \!\!\!\!
\end{equation}
where $\boldsymbol{V}_n$ satisfies $\boldsymbol{V}_n \cup \boldsymbol{V}_{n'} \subseteq \mathcal{U} \setminus \mathcal{S}_i \setminus \{ i \}$ for $1 \leq n,n' \leq \frac{N}{2}$ and $n + n' = \frac{N}{2}+1$.


An example is given in Table~\ref{tab:antithetical_sample} for $M\!=\!6$, $i\!=\!1$ and $N\!=\!8$.
Assume the optimization following Euqation~(\ref{eq:S_selection2}) achieves $\mathcal{S}_1 \!=\! \{2,3\}$.
The enumeration of $\boldsymbol{S} \subseteq \mathcal{S}_1$ following Equation~(\ref{eq:eff_shap_feature_contribution}) involves $\boldsymbol{S}_1 \!=\! \varnothing$, $\boldsymbol{S}_2 \!=\! \{2\}$, $\boldsymbol{S}_3 \!=\! \{3\}$ and $\boldsymbol{S}_4 \!=\! \{2,3\}$.
\Algnameabbr{} randomly samples the subsets of non-contributive features $\boldsymbol{V}_1 \!=\! \{5\}$ and $\boldsymbol{V}_2 \!=\! \{5,6\}$ from $\mathcal{U} \setminus \mathcal{S}_1 \setminus \{ 1 \} \!=\! \{4,5,6\}$; then adopts the AS to have $\boldsymbol{V}_3 \!=\! \mathcal{U} \setminus \mathcal{S}_1 \setminus \{ 1 \} \setminus \boldsymbol{V}_2 \!=\! \{ 4 \}$ and $\boldsymbol{V}_4 \!=\! \mathcal{U} \setminus \mathcal{S}_1 \setminus \{ 1 \} \setminus \boldsymbol{V}_1 \!=\! \{4,6\}$; finally estimates $\hat{\phi}_1$ following Equation~(\ref{eq:eff_shap_feature_contribution}).

\subsection{Algorithm of \Algnameabbr{}}

The configuration and pseudo code of \Algnameabbr{} are given in Figure~\ref{fig:eff_shap} and Algorithm~\ref{alg:eff_shap}, respectively.
To be concrete, \Algnameabbr{} receives a DNN model $f$ and feature value $x_1, \cdots, x_M$, and outputs the contributions $\hat{\phi}_1, \cdots, \hat{\phi}_M$ of feature $1, \cdots, M$, respectively.
For each feature $i$, \Algnameabbr{} first calculates its cross-contribution $\eta_{i,j}$ with other features $j \neq i \in \mathcal{U}$ following Equation~(\ref{eq:cross_contribution})~(Line 2); then greedily selects the contributive cooperators $\mathcal{S}_i$ following Equation~(\ref{eq:S_selection2})~(Line 3) to maximize the cross-contribution $\sum_{j \in \mathcal{S}_i} \eta_{i,j}$; finally estimates the contribution of feature $i$~(Line 4) throughout the coalitions of contributive cooperators $\boldsymbol{S} \subseteq \mathcal{S}_i$ following Equation~(\ref{eq:eff_shap_feature_contribution}).


\Algnameabbr{} enables the estimation of feature contribution to avoid the enumeration of all possible feature coalitions.
In this way, the estimation complexity can be significantly reduced from the brute-force complexity $T[ \phi_i (f_v, \mathcal{U}) ] = O(2^M)$ to $T(\hat{\phi}_i) = O(N)$, where $N \ll 2^M$.
Moreover, the estimation process for the $M$ features can execute independently without dependency on each other, which can be deployed on distributed systems for speeding up.

\begin{algorithm} 
\caption{\Algname{} (\Algnameabbr{})}
\label{alg:eff_shap}
{\bfseries Input:} DNN model $f$, input values $\boldsymbol{x} = [x_1, \cdots, x_M]$.\\
{\bfseries Output:}\mbox{Estimation value of feature contribution $\!\hat{\phi}_1, \! ..., \! \hat{\phi}_M$.}

\begin{algorithmic}[1]

\FOR{$i = 1, 2, \cdots, M$}

\STATE Estimate $\eta_{i,j}$ following Equation~(\ref{eq:cross_contribution}) for $j \!\in\! \mathcal{U} \!\setminus\! \{ i \}$.

\STATE Select contributive cooperators $\mathcal{S}_i$ via Equation~(\ref{eq:S_selection2}).

\STATE Estimate feature contribution $\hat{\phi}_i$ by Equation~(\ref{eq:eff_shap_feature_contribution}).


\ENDFOR

\end{algorithmic}
\end{algorithm}


\begin{table}\small 
\setlength{\abovecaptionskip}{1mm}
    \centering
    \caption{An example of antithetical sampling.}
    \begin{tabular}{c|c|c|c|c|c}
    \hline
        $\mathcal{U}$ & $i$ & $\mathcal{S}_i$ & 
        $n$ & 
        $\boldsymbol{S}_n$ & $\boldsymbol{V}_n$ \\
    \hline
        \multirow{4}{*}{\{ 1,2,3,4,5,6 \}} & \multirow{4}{*}{ 1 } & \multirow{4}{*}{\{ 2,3 \}} & 1 & $\varnothing$ & \{ 5 \} \\
        & & & 2 & \{2\} & \{ 5,6 \} \\
        & & & 3 & \{3\} & \{4\} \\
        & & & 4 & \{2,3\} & \{ 4,6 \} \\
    \hline
    \end{tabular}
    \label{tab:antithetical_sample}
\vspace{-3mm}
\end{table}

\subsection{Time Consumption Analysis}

In this section we analyze the time consumption of \Algnameabbr{}.
Generally, the time consumption of backward and forward process of DNNs is much more than that of other operators.
Hence, we only focus on the time-cost of model forward or backward processes in \Algnameabbr{}, and ignore the arithmetic and comparison operators in our analysis.
According to Algorithm~\ref{alg:eff_shap}, \Algnameabbr{} has one backward process to calculate the gradient $\nabla^2_{i,j} f_v(\mathcal{U}) + \nabla^2_{j,i} f_v(\mathcal{U})$ for the estimation of cross-contribution, and $2^{|\mathcal{S}_i|} \!\times\! 2 = N$ forward processes for the estimation of feature contribution following Equation~(\ref{eq:eff_shap_feature_contribution}).
Hence, the estimation for a single feature $\hat{\phi}_i$ has the time consumption given by
\begin{equation}
\setlength\abovedisplayskip{1mm}
\setlength\belowdisplayskip{1mm}
\label{eq:time_complexity}
T_{\text{\Algnameabbr{}}} \approx t_{\text{backward}} + N t_{\text{forward}},
\end{equation}
where $t_{\text{backward}}$ and $t_{\text{forward}}$ denote the time consumption of model backward and forward process, respectively.

Considering the interpretation of a model $f$ that has $M$ input features, the overall time cost increases to $M$ times if the $M$ features are processed consecutively; 
or we can reduce the time cost by simultaneously processing the $M$ input features based on the parallel structure.
We consider the first case in our experiments because this work focuses on algorithmic acceleration instead of the engineering trick.
To indicate the running speed of \Algnameabbr{}, we analyze the algorithmic throughput in Section~\ref{sec:exp_throughput}.

\section{Experiment}

In this section, we conduct experiments to evaluate \Algnameabbr{} by answering the following research questions.
In comparison to state-of-the-art baseline methods, 
\textbf{RQ1}: Does \Algnameabbr{} provide more accurate explanation by comparing with GT-Shapley value?
\textbf{RQ2}: Does \Algnameabbr{} provide more faithful explanations?
\textbf{RQ3}: Does \Algnameabbr{} run faster than the baseline methods?
\textbf{RQ4}: Does the contributive cooperator selection following Equation~(\ref{eq:S_selection2}) contribute to \Algnameabbr{}?





\subsection{Experiment Setup}

We provide the details about benchmark datasets, baseline methods and the experiment pipeline in this section.

\noindent
\textbf{Dataset}: The experiments involve Census Income, German Credit and Cretio datasets from the areas of social media, finance and recommender systems, respectively.
More details about the datasets are provided in Appendix~\ref{sec:appendix_dataset}.


\noindent
\textbf{Baseline Methods}: \Algnameabbr{} is compared with five state-of-the-art baseline methods of Shapley value estimation, including Kernel-SHAP~(KS)~\cite{lundberg2017unified}, Kernel-SHAP with Welford algorithm~(KS-WF)~\cite{covert2021improving}, Kernel-SHAP with Pair Sampling~(KS-Pair)~\cite{covert2021improving}, Permutation Sampling~(PS) \cite{mitchell2021sampling} and Antithetical Permutation Sampling~(APS)~\cite{lomeli2019antithetic}.
More details about the baseline methods are provided in Appendix~\ref{sec:appendix_baseline}.


\noindent
\textbf{Implementation Details}:
The experiments on each dataset follow the pipeline of \emph{model training}: training the DNN model; \emph{interpretation benchmark}: adopting the brute-force algorithm to calculate the exact Shapley value as the ground truth explanation for the evaluation; and \emph{interpretation evaluation}: evaluating the interpretation methods.
We give the details about each step in Appendix~\ref{sec:appendix_implementation_details}.


\begin{figure*}
\setlength{\abovecaptionskip}{1mm}
\setlength{\belowcaptionskip}{-5mm}
\centering
\subfigure[Census Income.]{
\centering
	\begin{minipage}[t]{0.3\linewidth}
		\includegraphics[width=0.99\linewidth]{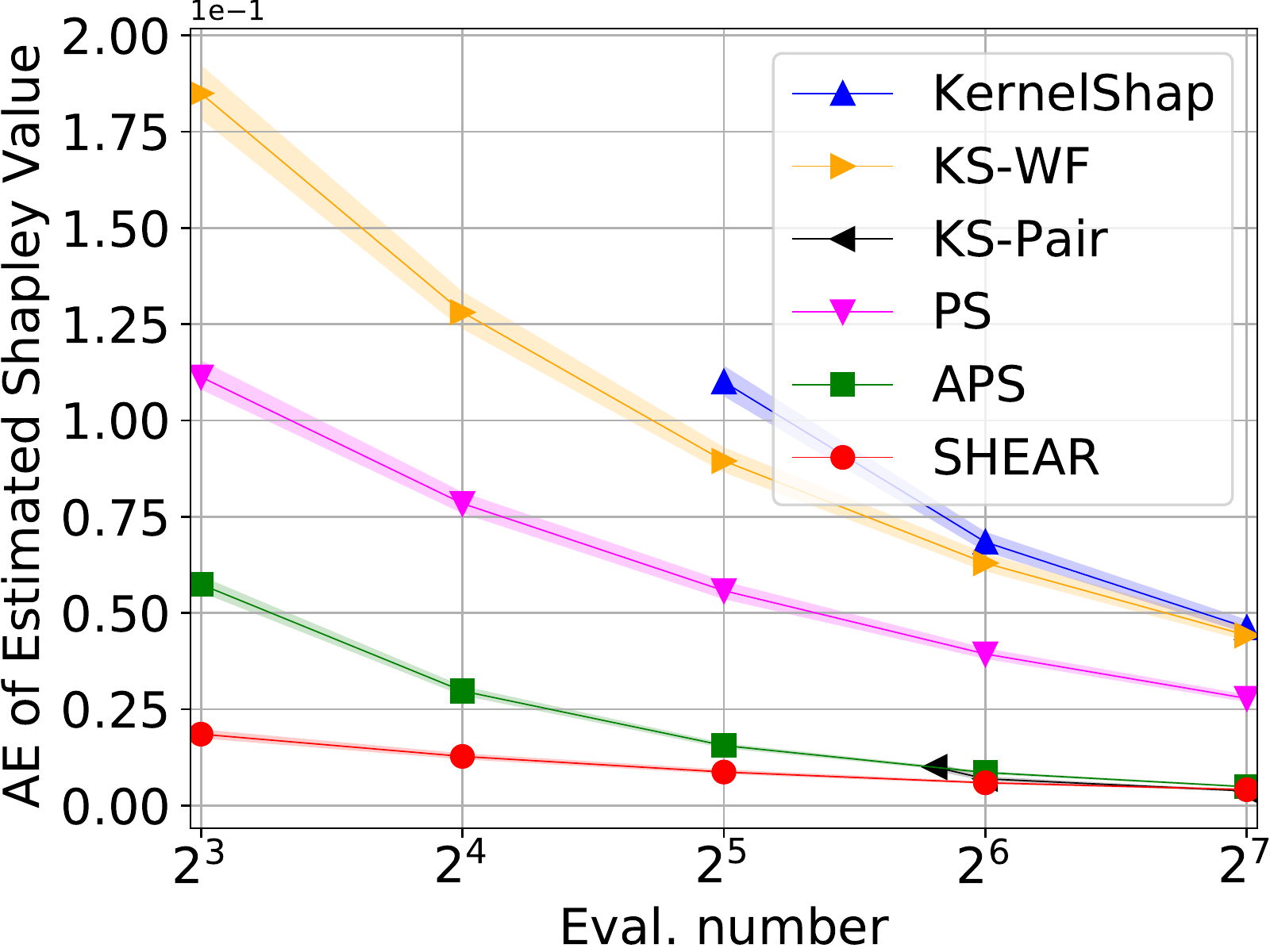}
	\end{minipage}%
}
$\quad$
\subfigure[German Credit.]{
\centering
	\begin{minipage}[t]{0.3\linewidth}
		\includegraphics[width=0.99\linewidth]{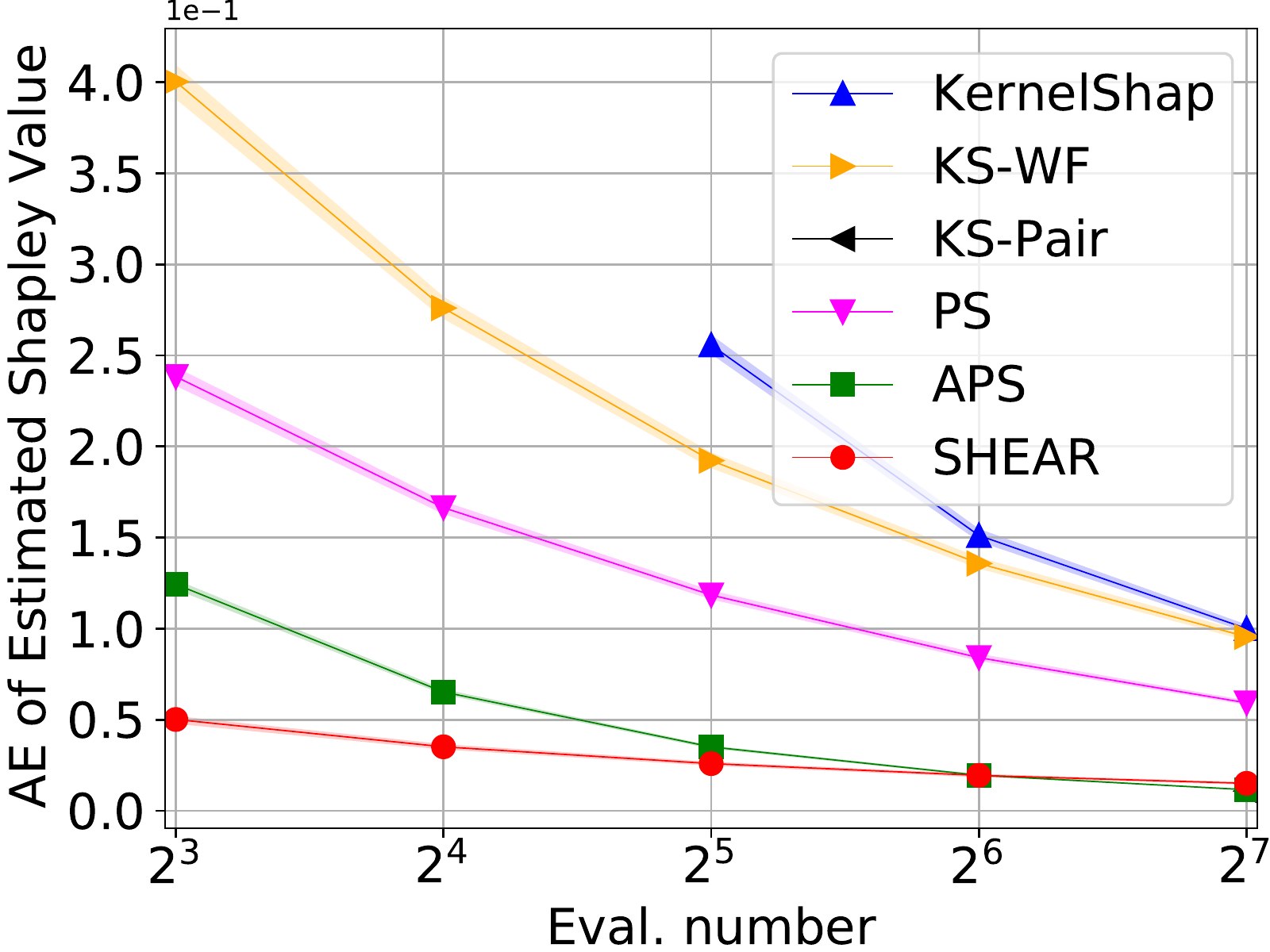}
	\end{minipage}%
}
$\quad$
\subfigure[Cretio.]{
\centering
	\begin{minipage}[t]{0.3\linewidth}
		\includegraphics[width=0.99\linewidth]{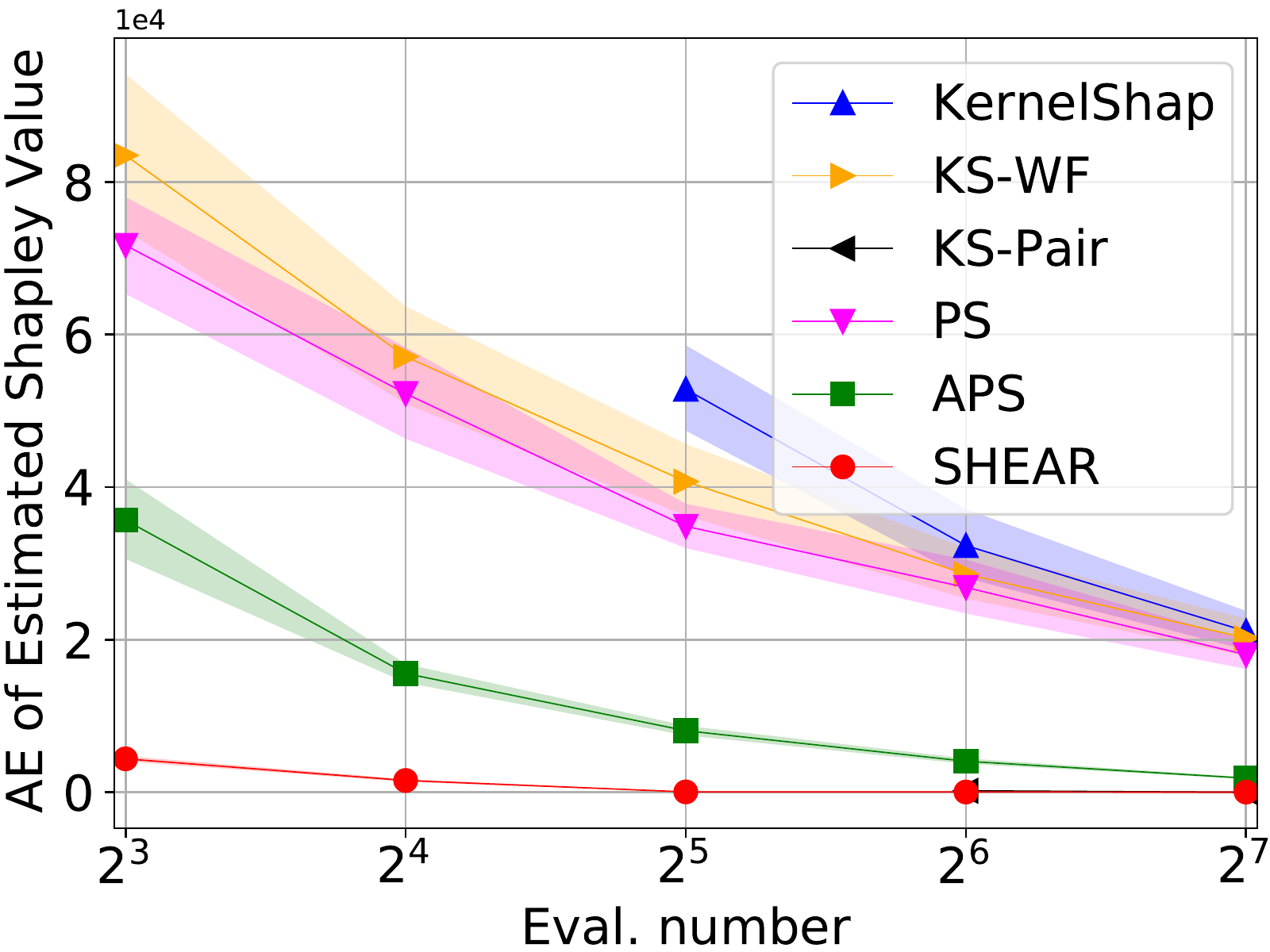}
	\end{minipage}
}
$\quad$
\subfigure[Census Income.]{
\centering
	\begin{minipage}[t]{0.3\linewidth}
		\includegraphics[width=0.99\linewidth]{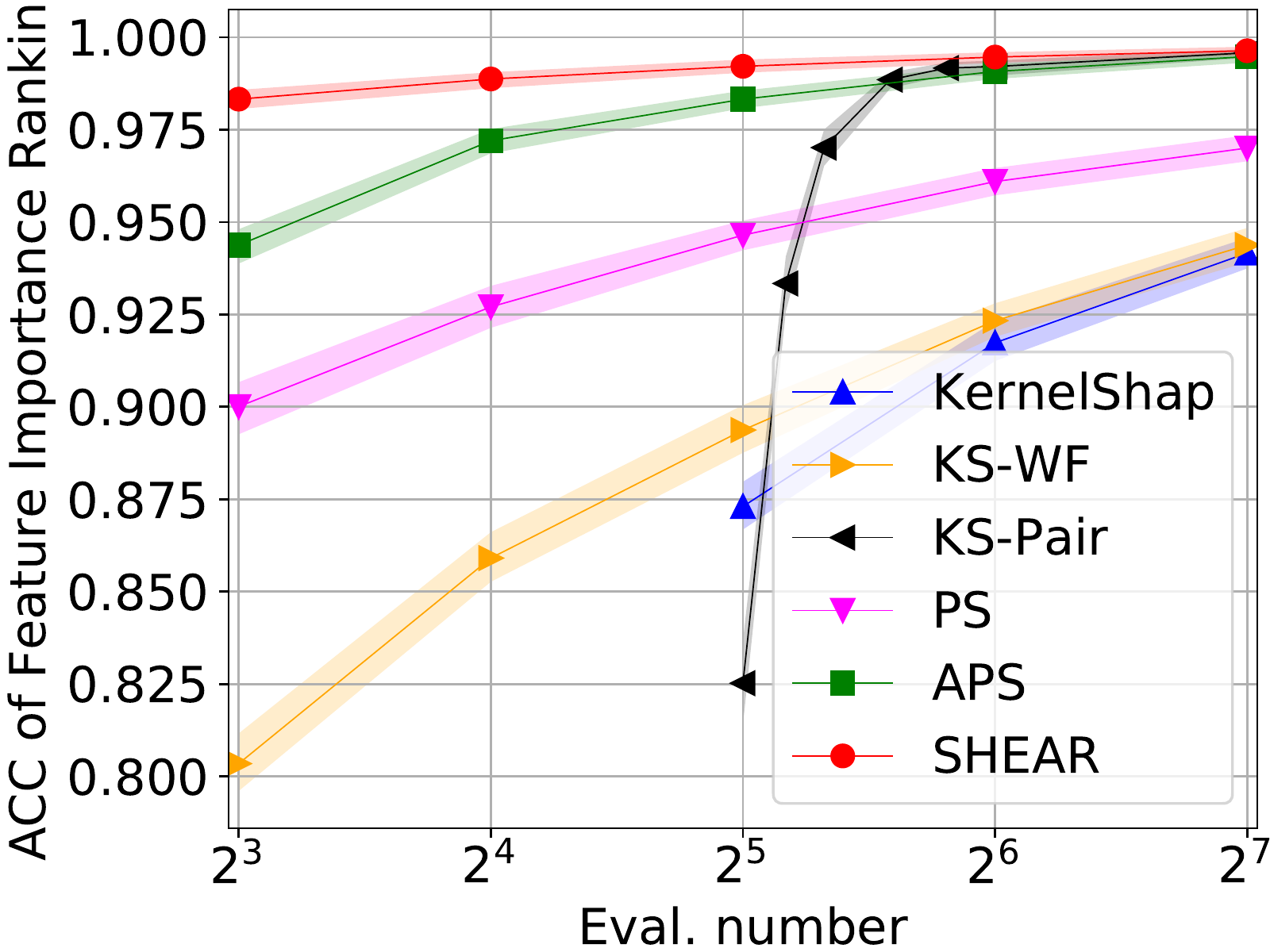}
	\end{minipage}%
}
$\quad$
\subfigure[German Credit.]{
\centering
	\begin{minipage}[t]{0.3\linewidth}
		\includegraphics[width=0.99\linewidth]{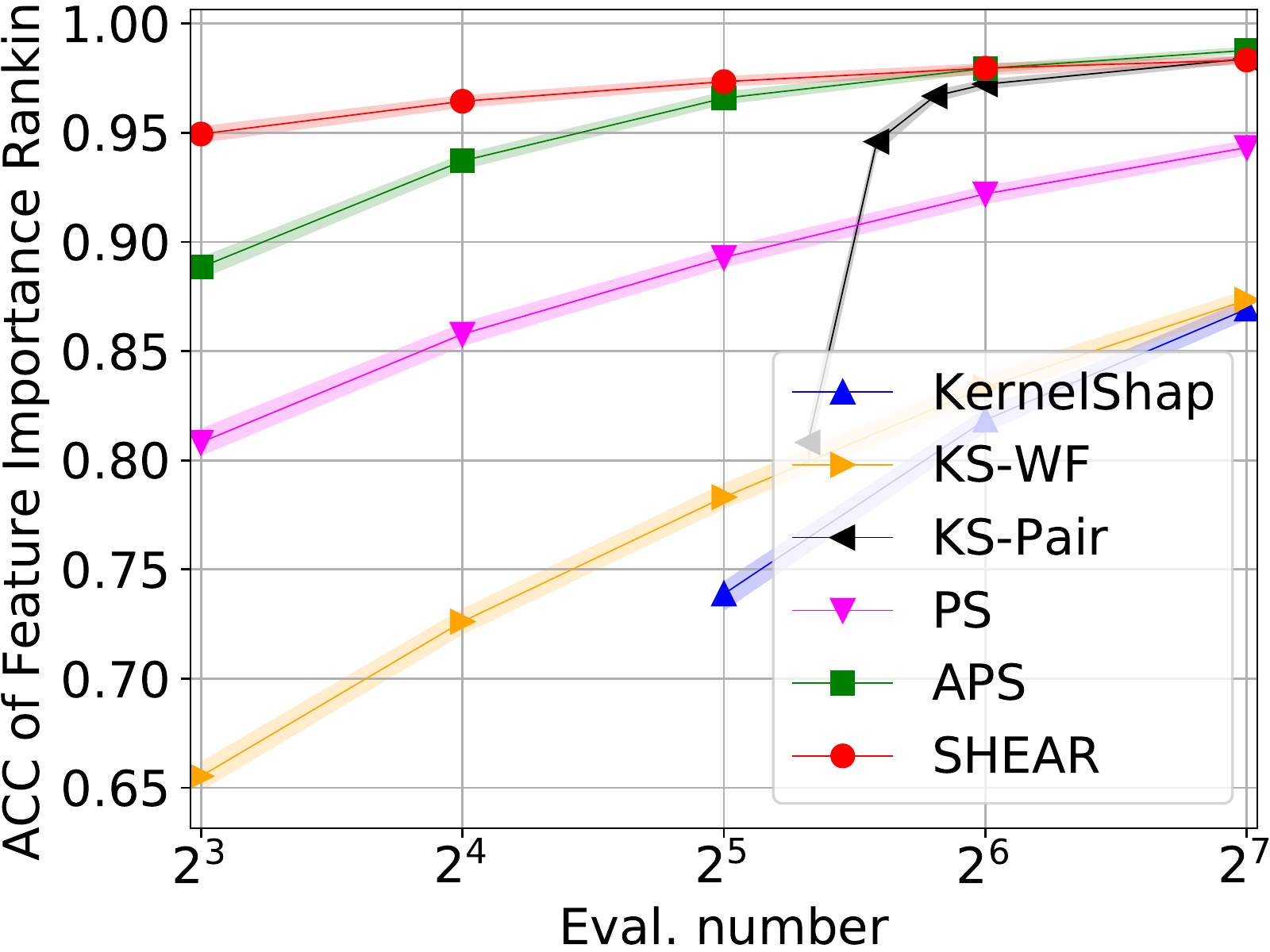}
	\end{minipage}%
}
$\quad$
\subfigure[Cretio.]{
\centering
	\begin{minipage}[t]{0.3\linewidth}
		\includegraphics[width=0.99\linewidth]{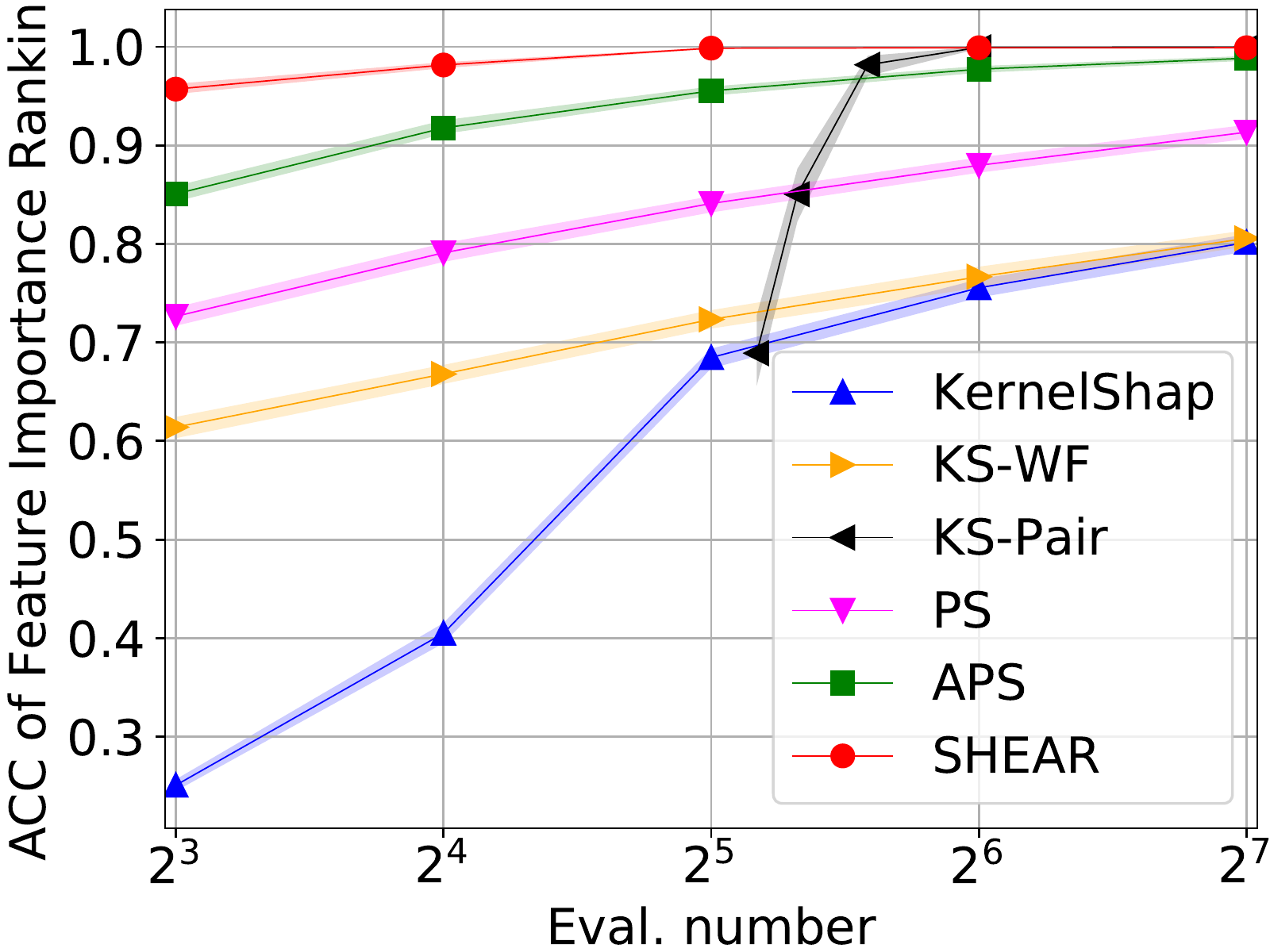}
	\end{minipage}
}
\caption{Absolute estimation error of the feature contribution on the (a) Census Income, (b) German Credit and (c) Cretio datasets; Accuracy of feature importance ranking on the (d) Census Income, (e) German Credit and (f) Cretio datasets.}
\label{fig:interpretation_performance_GT}
\end{figure*}


\subsection{Evaluation with GT-Shapley Value (RQ1)}
\label{sec:GT_eval_metric}


In this section, we evaluate the interpretation methods via taking the GT-Shapley value as the ground truth explanation, and taking the GT-Shapley value ranking as the ground truth feature importance ranking.
Specifically, we consider two metrics to evaluate the interpretation performance: the absolute estimation error~(AE) and accuracy of feature importance ranking~(ACC).
Given the GT-Shapley value $\phi_1, \cdots, \phi_M$ from the brute-force algorithm, the metrics AE and ACC are formulated as follows:

\noindent
\textbf{Absolute estimation error}: For the estimation value of feature contribution $\hat{\phi}_1, \cdots, \hat{\phi}_M$, the absolute estimation error is given by
\begin{equation}
\setlength\abovedisplayskip{-1mm}
\setlength\belowdisplayskip{-1mm}
\mathrm{AE} = \sum_{i=1}^M \mid \phi_i - \hat{\phi}_i \mid.
\nonumber
\end{equation}

\noindent
\textbf{Accuracy of feature importance ranking}: Let $r_1, \cdots, r_M$ and $\hat{r}_1, \cdots, \hat{r}_M$ denote the descending ranking of $\phi_1, \cdots, \phi_M$ and $\hat{\phi}_1, \cdots, \hat{\phi}_M$, respectively.
The accuracy of feature importance ranking~\cite{wojtas2020feature} can be calculated as follows:
\begin{equation}
\setlength\abovedisplayskip{1mm}
\setlength\belowdisplayskip{1mm}
\begin{aligned}
\mathrm{ACC} = \frac{\sum_{m=1}^M \frac{\mathbf{1}_{\hat{r}_m = r_m}}{m}}{\sum_{m=1}^M \frac{1}{m}},
\nonumber
\end{aligned}
\end{equation}
where the factor $\frac{1}{m}$ enables the important features contribute more to the accuracy; and the factor $(\sum_{m=1}^M \frac{1}{m})^{-1}$ normalizes the accuracy such that $0 \leq \text{ACC} \leq 1$.

We give the absolute estimation error versus the times of model evaluation in Figures~\ref{fig:interpretation_performance_GT}~(a)-(c), and the accuracy of feature importance ranking  in Figures~\ref{fig:interpretation_performance_GT}~(d)-(f); we also plot the error bar to show the standard deviation of each method across multiple rounds.
According to the experimental results, we have the following observations: 
\begin{itemize}[leftmargin=10pt, topsep=0pt]
\setlength{\parskip}{1pt}
\setlength{\parsep}{0pt}
\setlength{\itemsep}{0pt}

    \item As the number of model evaluation grows, we observe less absolute estimation error and more accurate feature importance ranking of each method on the three datasets.
    The reason is that more model evaluations enable each method to converge to the GT-Shapley value.
    
    \item With unified model evaluation time, \Algnameabbr{} achieves the least absolute estimation error and the most accurate feature importance ranking.
    The experimental results indicate the efficient utilization of model evaluations enables \Algnameabbr{} to adapt to real-world application, where the available times of model evaluation is far from enough to enumerate all possible input feature coalitions.

    \item Compared with Kernel-SHAP, KS-WF and PS, \Algnameabbr{} gets a WIN-WIN situation, where \Algnameabbr{} shows MORE accurate explanation using FEWER model evaluations.
    
    \item \Algnameabbr{} shows the least standard deviation of the absolute error and accuracy compared with the baseline methods due to the fact that \Algnameabbr{} adopts greedy search to select the contributive cooperators without random initialization.
    The stable interpretation performance of \Algnameabbr{} indicates its robustness towards different models trained on different datasets.

\end{itemize}



\begin{figure*}
\setlength{\abovecaptionskip}{0mm}
\setlength{\belowcaptionskip}{-6mm}
\centering
\subfigure[Census Income.]{
\centering
	\begin{minipage}[t]{0.3\linewidth}
		\includegraphics[width=0.99\linewidth]{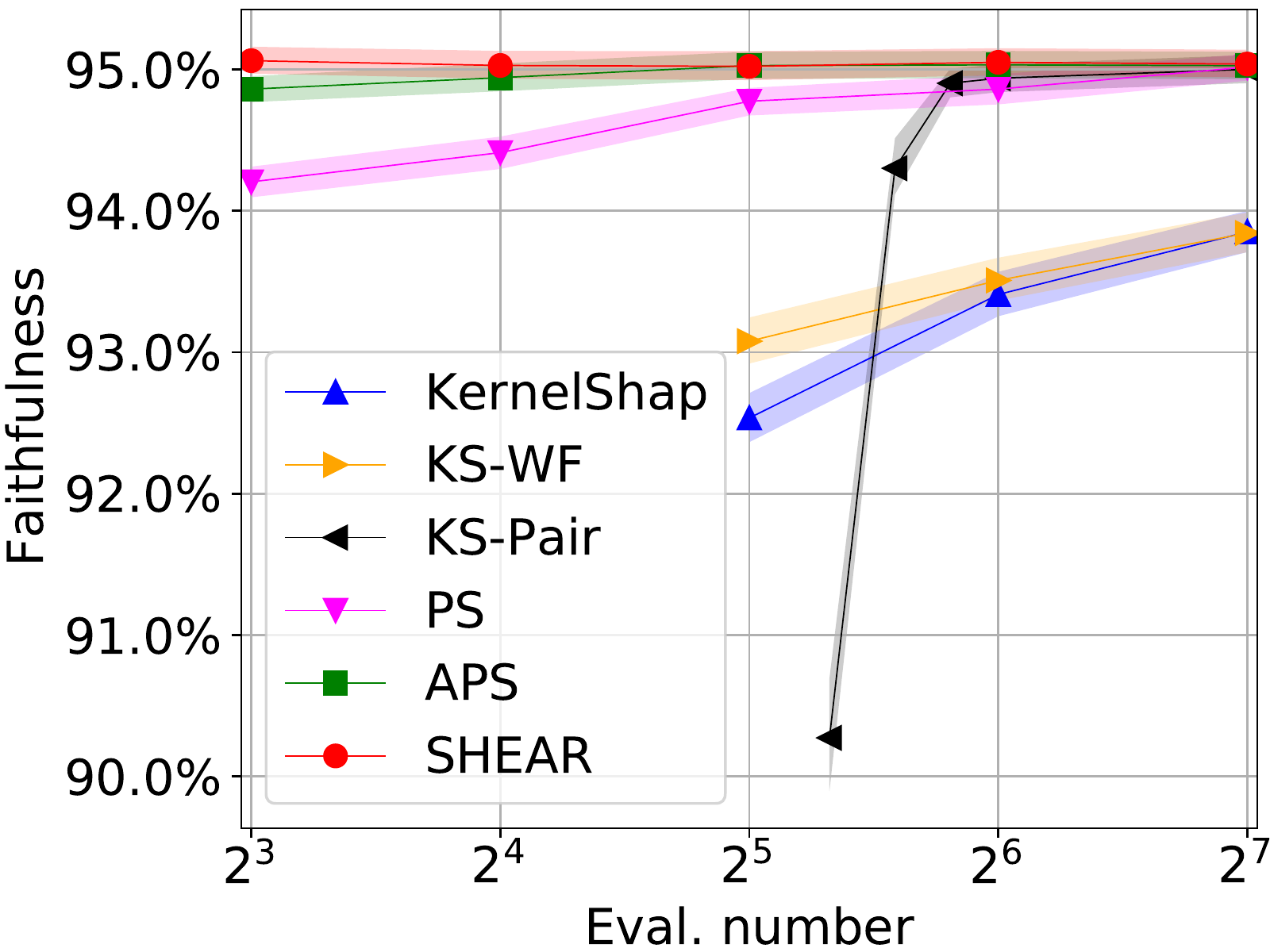}
	\end{minipage}%
}
$\quad$
\subfigure[German Credit.]{
\centering
	\begin{minipage}[t]{0.3\linewidth}
		\includegraphics[width=0.99\linewidth]{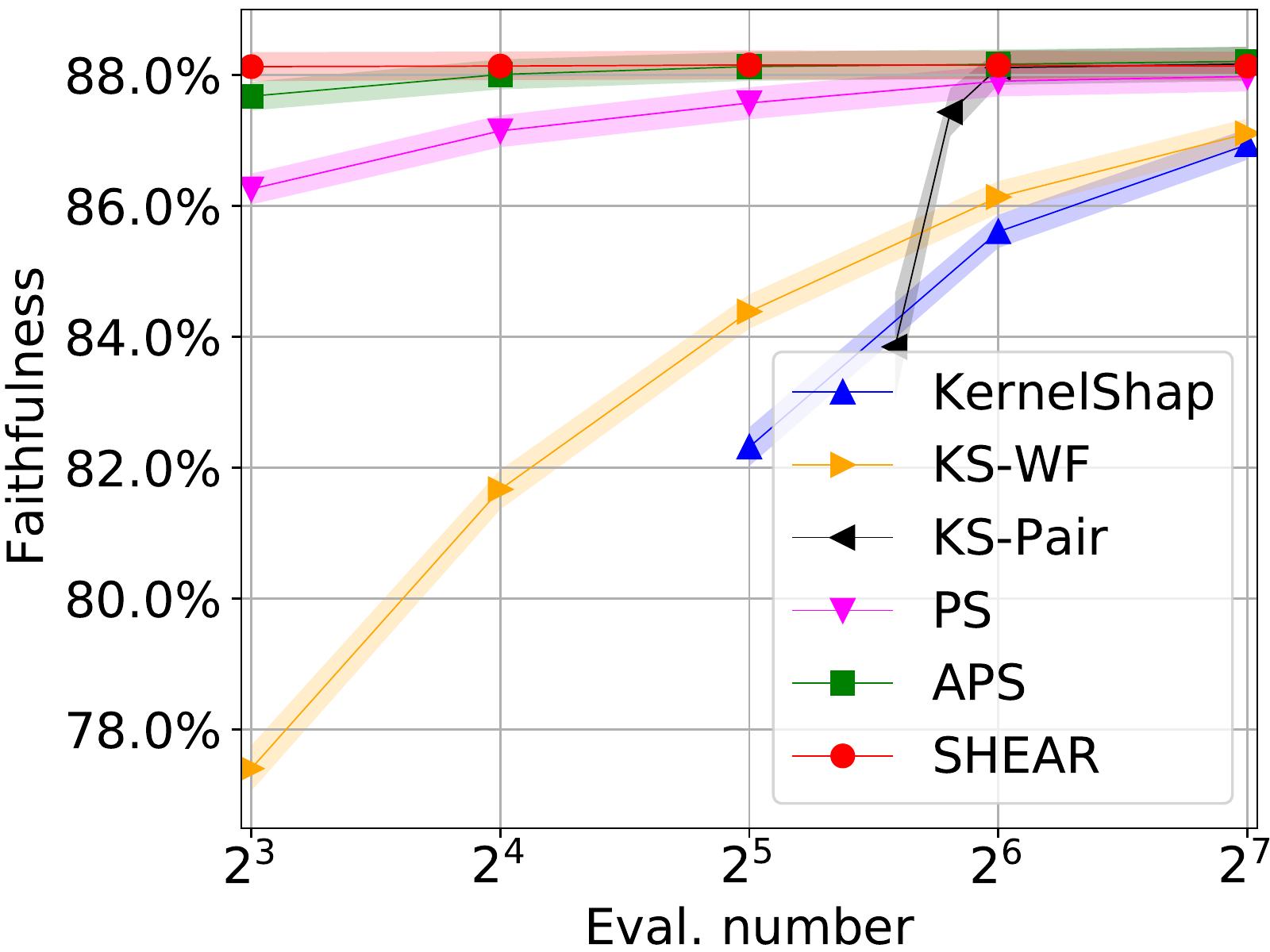}
	\end{minipage}
}
$\quad$
\subfigure[Cretio.]{
\centering
	\begin{minipage}[t]{0.3\linewidth}
		\includegraphics[width=0.99\linewidth]{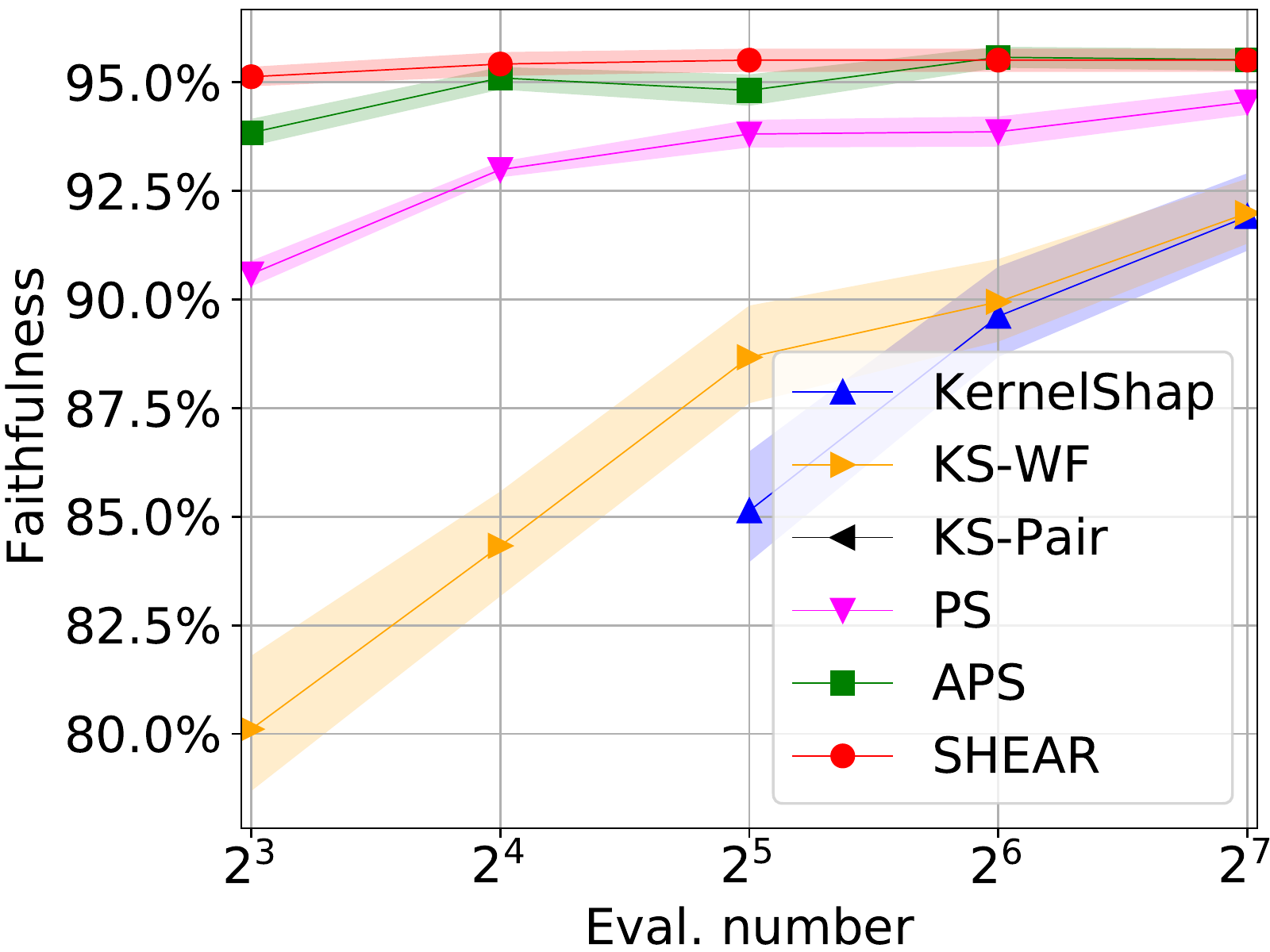}
	\end{minipage}%
}
\subfigure[Census Income.]{
\centering
	\begin{minipage}[t]{0.3\linewidth}
		\includegraphics[width=0.99\linewidth]{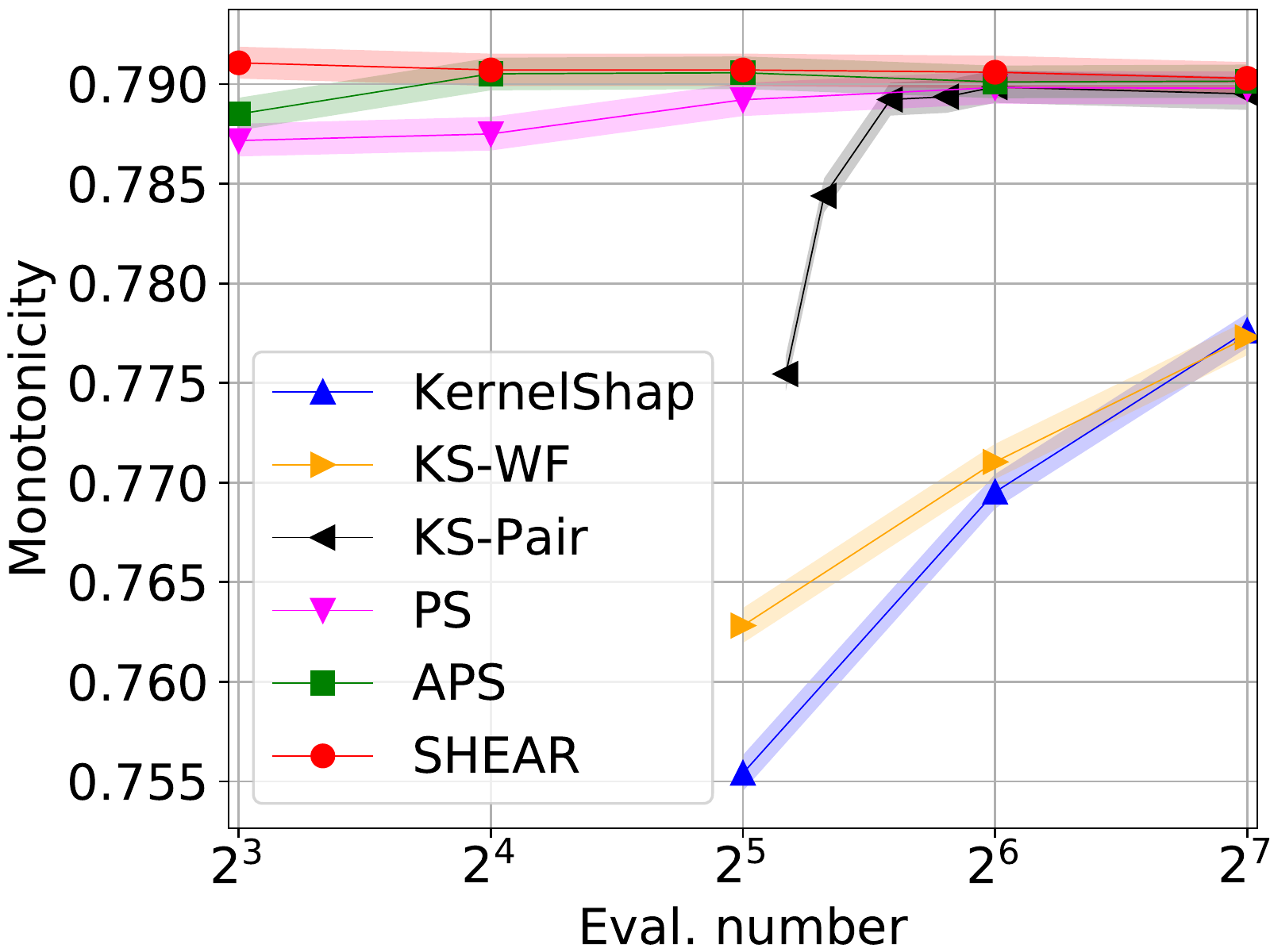}
	\end{minipage}%
}
$\quad$
\subfigure[German Credit.]{
\centering
	\begin{minipage}[t]{0.3\linewidth}
		\includegraphics[width=0.99\linewidth]{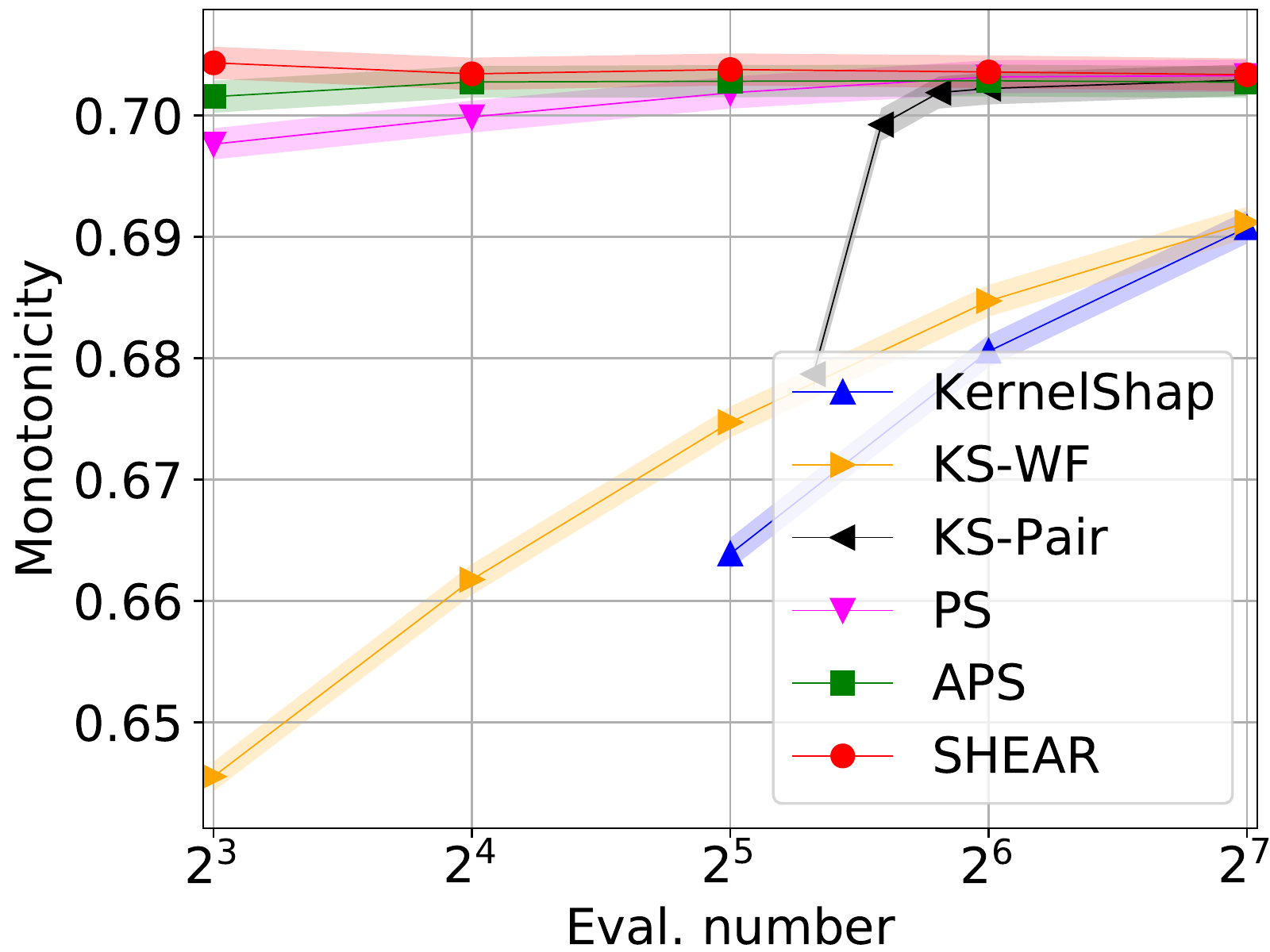}
	\end{minipage}
}
$\quad$
\subfigure[Cretio.]{
\centering
	\begin{minipage}[t]{0.3\linewidth}
		\includegraphics[width=0.99\linewidth]{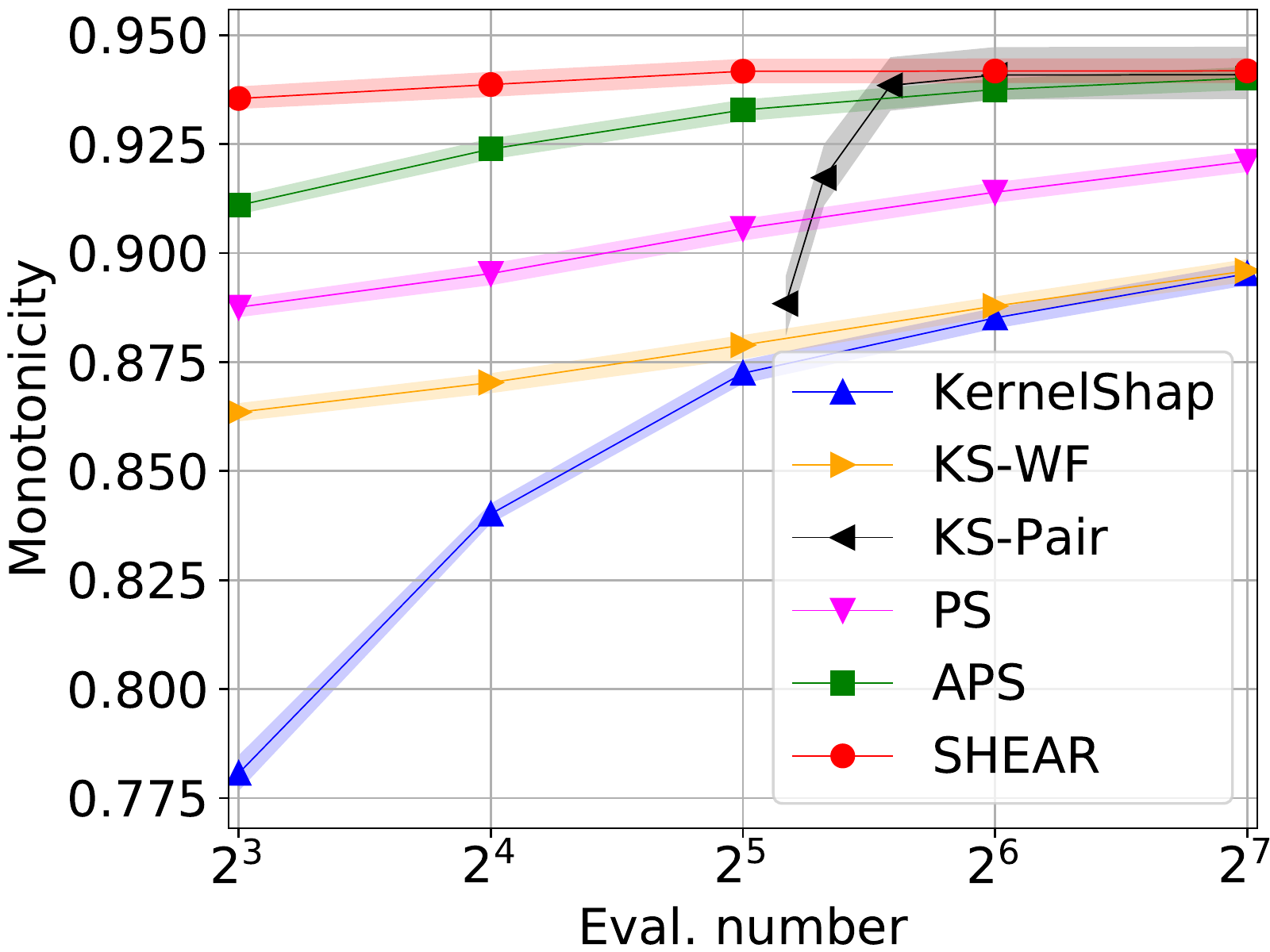}
	\end{minipage}%
}
\caption{Faithfulness of the explanation on the (a) Census Income, (b) German Credit and (c) Cretio datasets; Monotonicity of the explanation on the (d) Census Income, (e) German Credit and (f) Cretio datasets.}
\label{fig:interpretation_performance_NGT}
\end{figure*}

\subsection{Evaluation with Model Perturbation (RQ2)}
\label{sec:NGT_eval_metric}

The evaluation methods with model perturbation are motivated by the common sense that important features have more impact on the prediction than trivial features.
Specifically, we follow the existing work~\cite{liu2021synthetic} to consider two metrics to evaluate the interpretation performance with model perturbation: \emph{Faithfulness} and \emph{Monotonicity}.

\noindent
\textbf{Faithfulness}: For the estimation value of feature contribution $\hat{\phi}_1, \cdots, \hat{\phi}_M$, Faithfulness computes the Pearson correlation coefficient~\cite{benesty2009pearson} between the feature contribution and preceding difference of the model,
\begin{equation}
\setlength\abovedisplayskip{1mm}
\setlength\belowdisplayskip{1mm}
\text{Faithful} \!=\! \text{Pearson} \Big( \! \big[ f_v \! (\mathcal{U}) -\! f_v \! (\mathcal{U} \setminus \{ i \}) \big]_{1\leq i\leq M}, \! [ \hat{\phi}_i ]_{1\leq i\leq M} \! \Big) \! ,
\nonumber
\end{equation}
where the Pearson correlation coefficient is given by: $\text{Pearson} (X, Y) = \frac{\mathrm{cov} (X,Y)}{\sigma_X \sigma_Y}$; $\mathrm{cov} (X,Y)$ denotes the co-variance of $X$ and $Y$; and $\sigma_X$ and $\sigma_Y$ denote the standard deviation of $X$ and $Y$, respectively. 
Larger Faithfulness indicates better interpretation.

\noindent
\textbf{Monotonicity}: Monotonicity evaluates the feature importance ranking without the ground-truth ranking.
It computes the marginal improvement of each feature ordered by the estimated feature contribution, and calculates the fraction of indices $i$ such that the marginal improvement of feature $i$ is greater than feature $i + 1$.
Monotonicity is given by
\begin{equation}
\setlength\abovedisplayskip{1mm}
\setlength\belowdisplayskip{1mm}
\text{Monotonicity} = \frac{1}{M-1} \sum_{i=0}^{M-2} \mathbf{1}_{\delta_i \geq \delta_{i+1}},
\nonumber
\end{equation}
where $\delta_i = f_v( \mathcal{W}_i \cup \{ i \}) - f_v( \mathcal{W}_i )$ indicates the marginal improvement of the top-$i$ important feature; and $\mathcal{W}_i = \{ j \in \mathcal{U} \setminus \{i\} \mid \phi_j \geq \phi_i \}$ denotes the features more important than the top-$i$ important feature.
Larger Monotonicity implies better feature importance ranking.

The Faithfulness of \Algnameabbr{} and baseline methods versus the times of model evaluation are shown in Figures~\ref{fig:interpretation_performance_NGT}~(a)-(c), and Monotonicity versus the times of model evaluation are give in Figures~\ref{fig:interpretation_performance_NGT}~(d)-(f).
The error bar is given in the figures to show the standard deviation of each method.
Overall, we have the following observations:
\begin{itemize}[leftmargin=10pt, topsep=0pt]
\setlength{\parskip}{1pt}
\setlength{\parsep}{0pt}
\setlength{\itemsep}{0pt}

    \item All methods achieve larger Faithfulness and Monotonicity as the times of model evaluation  grows due to the fact that more evaluations provide more information for interpreting the prediction.
    
    \item \Algnameabbr{} achieves larger Faithfulness and Monotonicity than baseline methods under the same times of model evaluation, which indicates the effectiveness of \Algnameabbr{}.
    
    \item \Algnameabbr{} shows consistently better performance than baseline methods as the times of model decreases.
    
    \item According to Figures~\ref{fig:interpretation_performance_GT} and~\ref{fig:interpretation_performance_NGT}, even though the curves of absolute error, accuracy of feature importance ranking, Faithfulness and Monotonicity are different, the interpretation performance ranking of \Algnameabbr{} and baseline methods indicated by the four metrics are consistent with each other.
    Hence, the effectiveness of \Algnameabbr{} can be demonstrated by multiple evaluation metrics.
    
\end{itemize}

\begin{figure*}
\setlength{\abovecaptionskip}{0mm}
\setlength{\belowcaptionskip}{-6mm}
\centering
\subfigure[Census Income.]{
\centering
	\begin{minipage}[t]{0.3\linewidth}
		\includegraphics[width=0.99\linewidth]{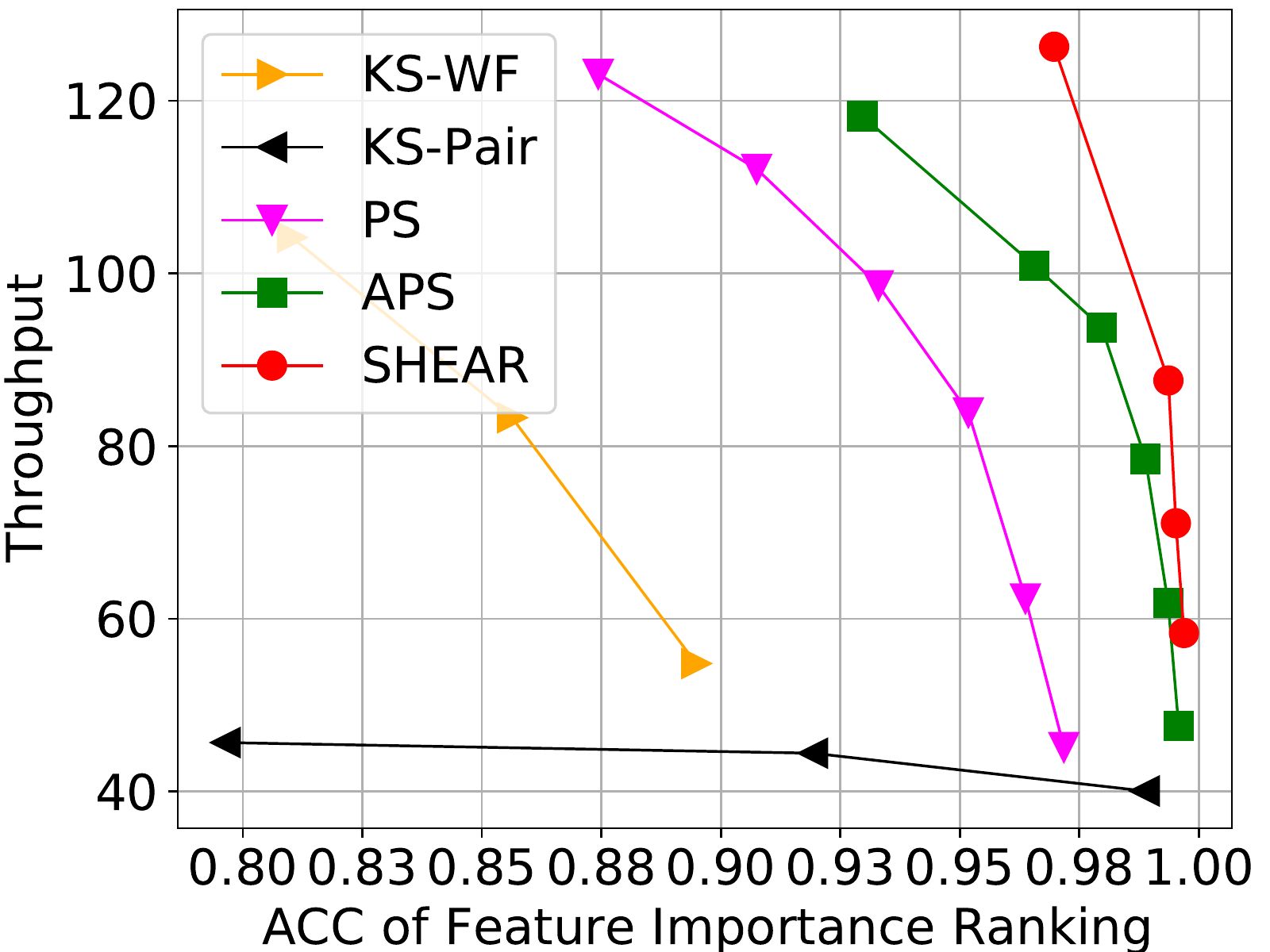}
	\end{minipage}%
}
$\quad$
\subfigure[German Credit.]{
\centering
	\begin{minipage}[t]{0.3\linewidth}
		\includegraphics[width=0.99\linewidth]{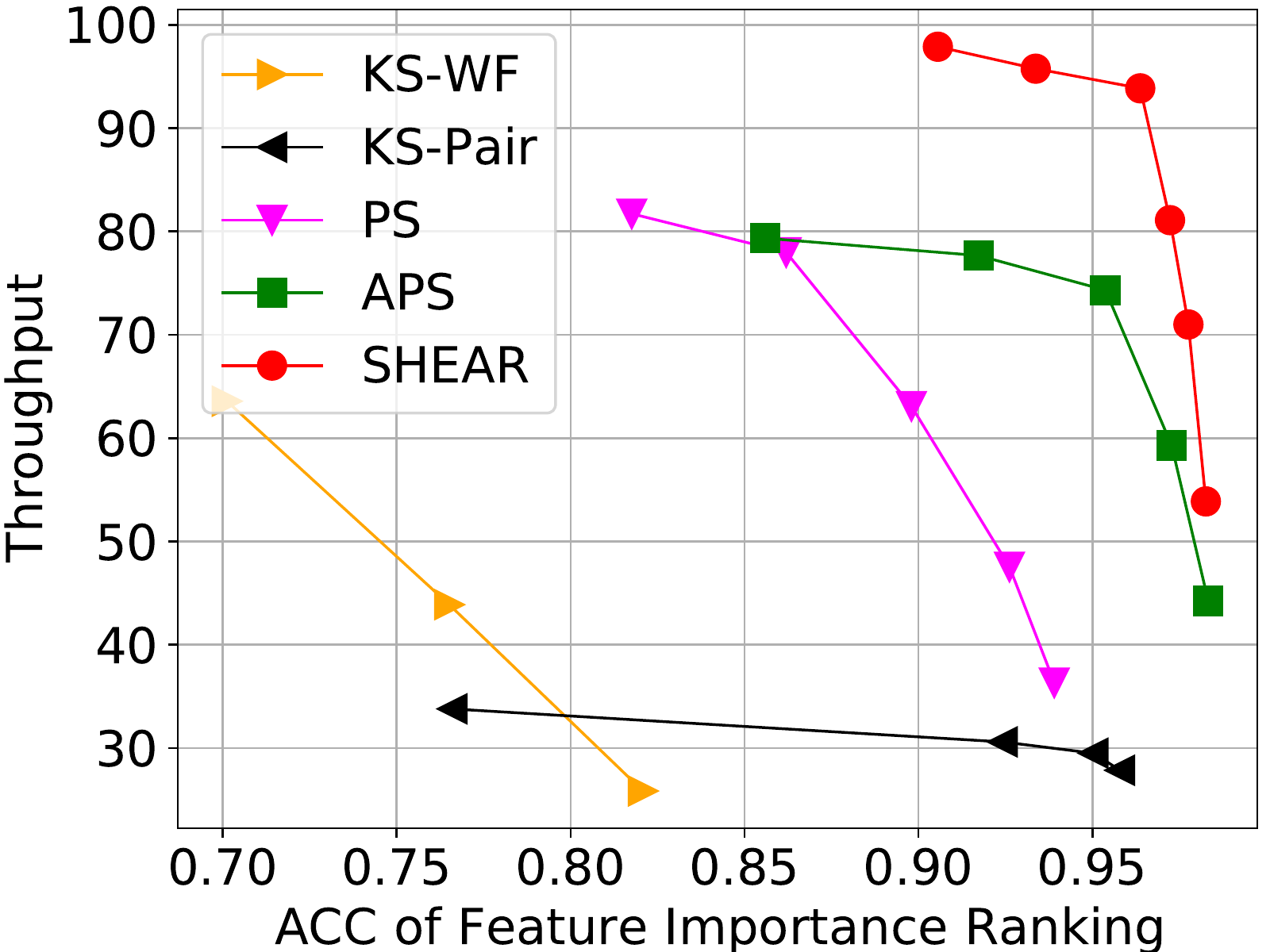}
	\end{minipage}%
}
$\quad$
\subfigure[Cretio.]{
\centering
	\begin{minipage}[t]{0.3\linewidth}
		\includegraphics[width=0.99\linewidth]{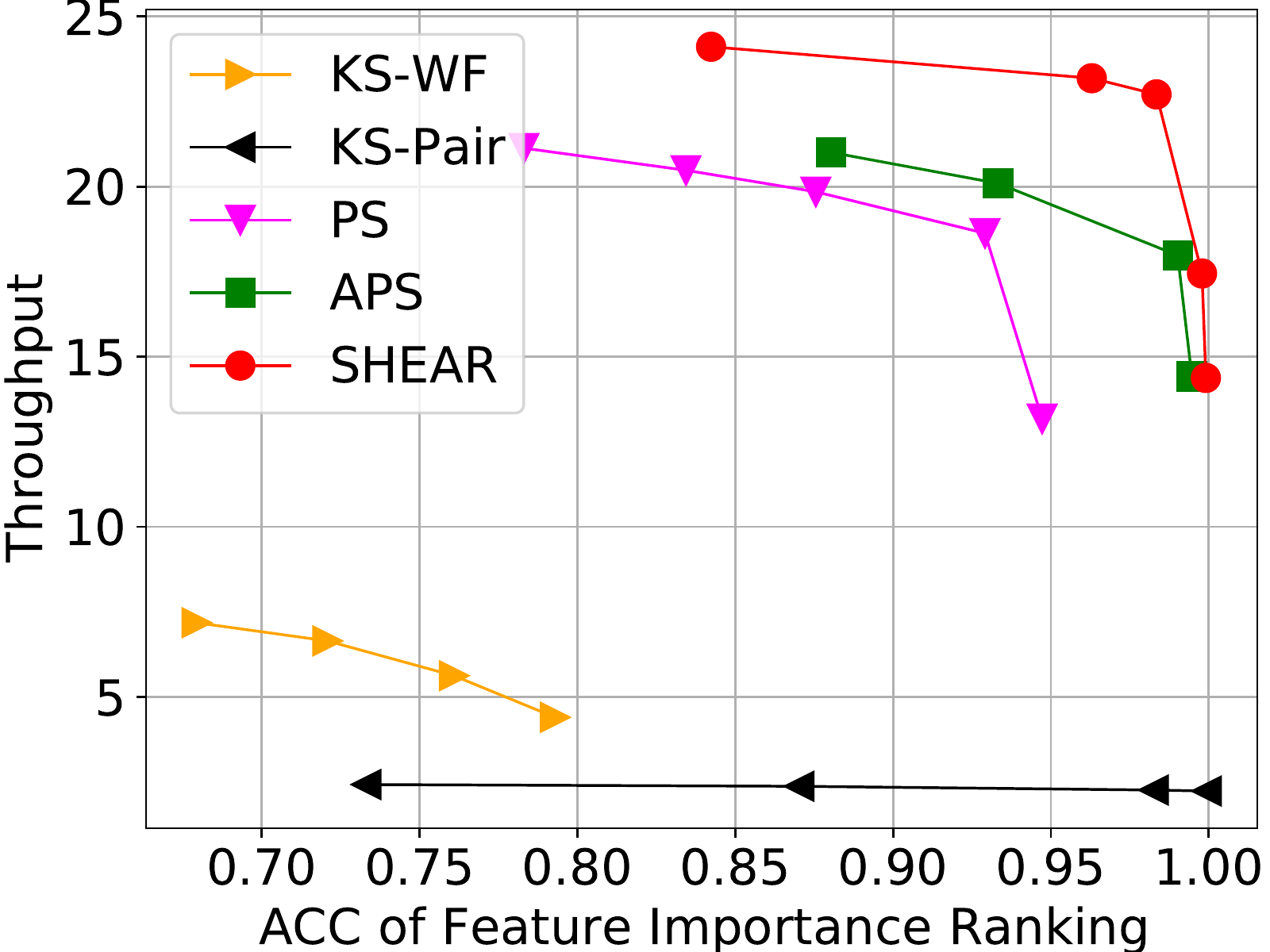}
	\end{minipage}
}
\subfigure[Ablation study.]{
\centering
	\begin{minipage}[t]{0.3\linewidth}
		\includegraphics[width=0.99\linewidth]{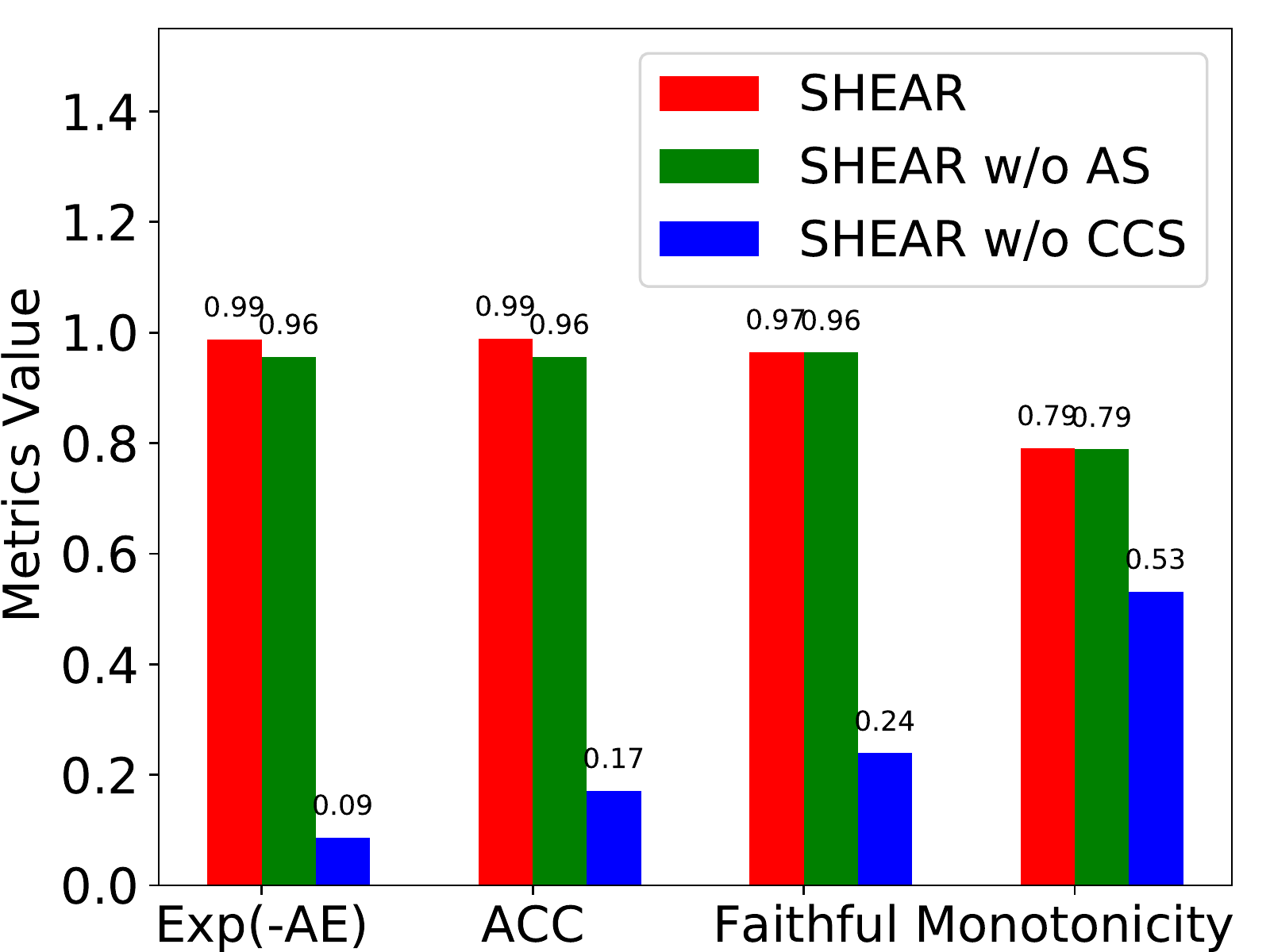}
	\end{minipage}%
}
\subfigure[Estimation error vs. Cross-contribution.]{
\centering
	\begin{minipage}[t]{0.335\linewidth}
		\includegraphics[width=0.88\linewidth]{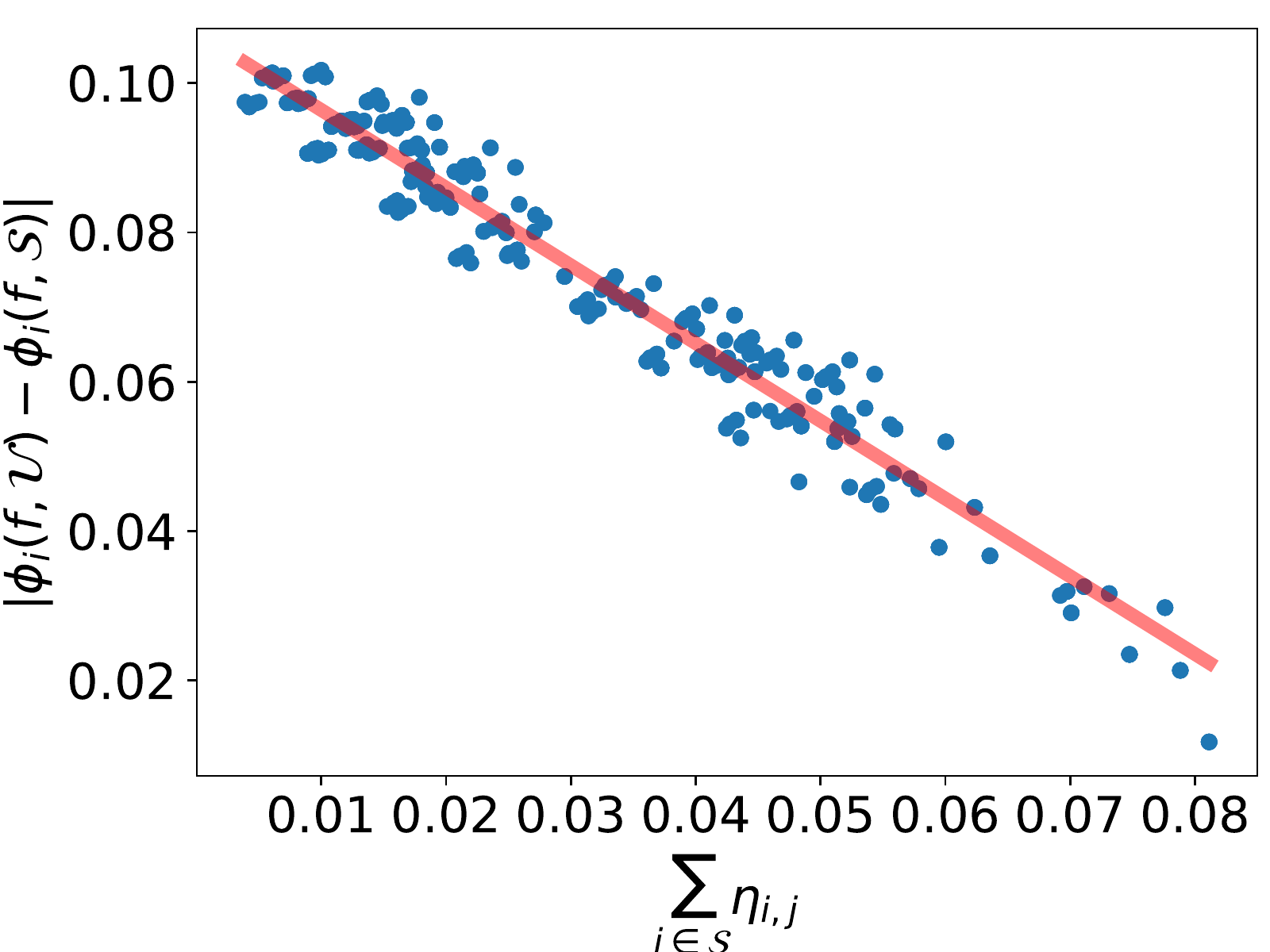}
	\end{minipage}%
}
\subfigure[Estimation error vs. Evaluation times.]{
\centering
	\begin{minipage}[t]{0.335\linewidth}
		\includegraphics[width=0.88\linewidth]{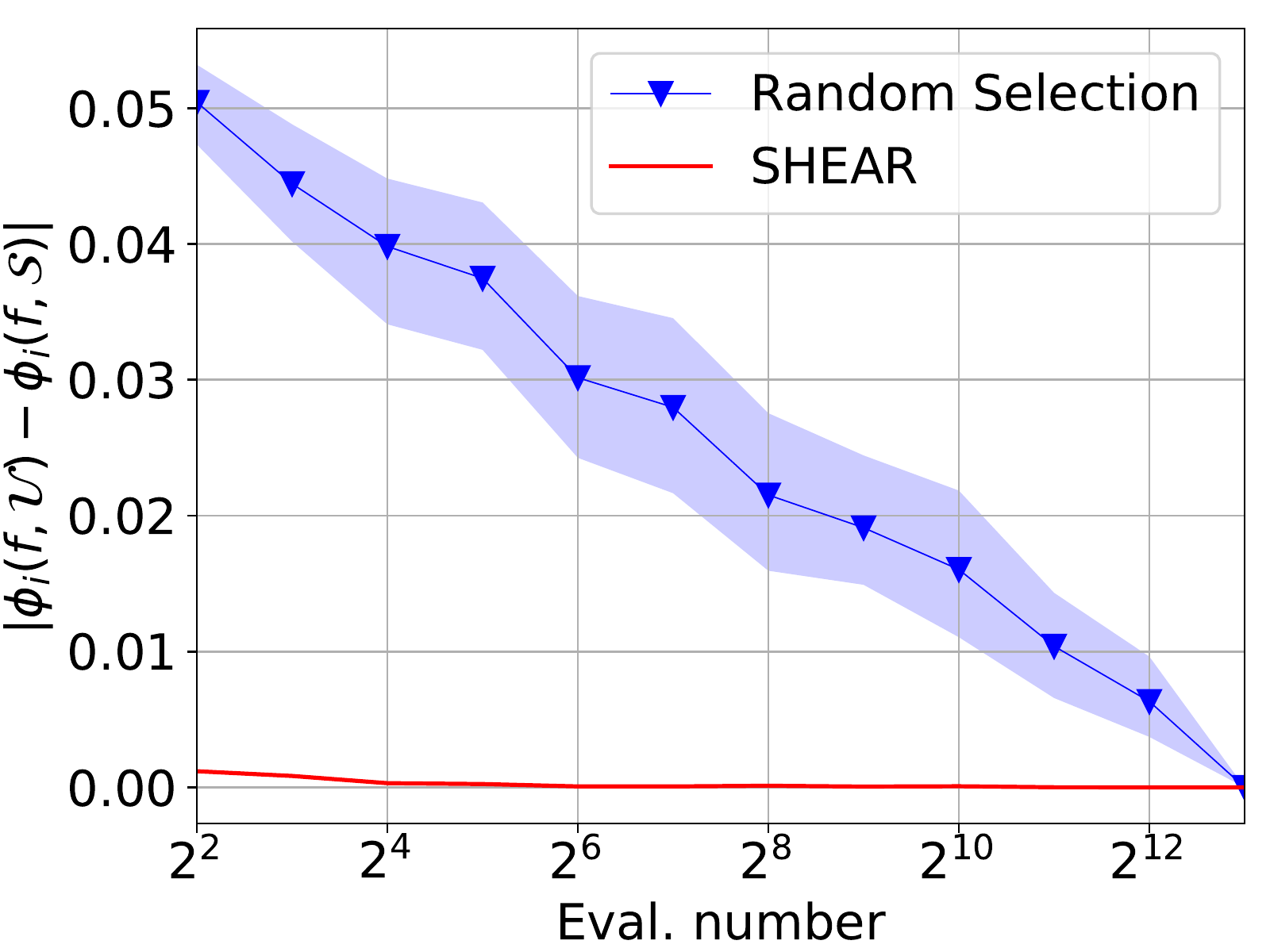}
	\end{minipage}
}
\caption{Algorithmic throughput vs. ACC of feature importance ranking on the (a) Census Income, (b) German Credit and (c) Cretio datasets; (d) \Algnameabbr{} vs. \Algnameabbr{} w/o AS vs. \Algnameabbr{} w/o CCS; (e) Absolute estimation error vs. Cross-contribution; (f) Absolute estimation error vs. Times of model evaluation.}
\label{fig:run_time}
\end{figure*}

\subsection{Throughput Evaluation~(RQ3)}
\label{sec:exp_throughput}

Algorithmic throughput is estimated by $\frac{N_{\text{test}}}{t_{\text{total}}}$, where $N_{\text{test}}$ and $t_{\text{total}}$ denote the testing instance number and the total time consumption of the interpreting process, respectively.
$N_{\text{test}}$ of the three datasets is given in Appendix~\ref{sec:appendix_dataset}, and $t_{\text{total}}$ is tested based on the physical computing infrastructure given in Appendix~\ref{sec:appendix_infrastructure}.
We plot the throughput versus the accuracy of feature importance ranking on the three datasets in Figures~\ref{fig:run_time}~(a)-(c), respectively, where we omit the curve of Kernel-SHAP due to its low accuracy observed from previous experiments. 
According to the experimental results, we have the following observations:
\begin{itemize}[leftmargin=10pt, topsep=0pt]
\setlength{\parskip}{1pt}
\setlength{\parsep}{0pt}
\setlength{\itemsep}{0pt}
    
    \item For all methods, the throughput reduces as the interpretation performance grows due to the fact that more accurate explanation depends on more model evaluations which spend more time in the model forward process.
    
    \item \Algnameabbr{} shows the most efficient interpretation, which has the highest accuracy when controlling the throughput and the most throughput when controlling the accuracy.
    
    \item Even though \Algnameabbr{} has additional time cost on the gradient estimation according to Equation~(\ref{eq:time_complexity}), it provides more accurate explanations than baseline methods, thus having better throughput and accuracy trade-off. 
    
\end{itemize}

\subsection{Ablation Study (RQ4)}

The effectiveness of \Algnameabbr{} mainly derives from the contributive cooperator selection following Equation~(\ref{eq:S_selection2}).
To prove this, we compare the interpretation performance of \Algnameabbr{}, \Algnameabbr{} without antithetical sampling (AS) and \Algnameabbr{} without cooperator selection (CCS) in Figure~\ref{fig:run_time}~(d), where $N \!=\! 16$.
We have the following observations:
\begin{itemize}[leftmargin=10pt, topsep=0pt]
\setlength{\parskip}{1pt}
\setlength{\parsep}{0pt}
\setlength{\itemsep}{0pt}

\item \Algnameabbr{} outperforms \Algnameabbr{} w/o AS, which demonstrates the effectiveness of antithetical sampling.

\item \Algnameabbr{} w/o CCS shows considerable interpretation degradation across the four evaluation metrics, which indicates the dominant contribution of CCS to \Algnameabbr{}.

\end{itemize}


We conduct the follow-up experiments to trace the contributive cooperator selection.
Specifically, for each instance in the testing dataset, we select a target feature $i \in \mathcal{U}$ to calculate the absolute error $|\phi_i (f, \mathcal{U}) - \phi_i (f, \mathcal{S} \cup \{ i \})|$ via the brute-force algorithm, and calculate the cross-contribution $\sum_{j \in \mathcal{S}} \eta_{i,j}$ following Equation~(\ref{eq:cross_contribution}), where $\mathcal{S} \!\subseteq\! \mathcal{U} \setminus \{ i \}$ is randomly selected satisfying $|\mathcal{S}| \!=\! \log \frac{N}{2}$.
For $N \!=\! 16$, we illustrate the numerical relationship between $|\phi_i (f, \mathcal{U}) - \phi_i (f, \mathcal{S} \cup \{ i \})|$ and $\sum_{j \in \mathcal{S}} \eta_{i,j}$ in Figure~\ref{fig:run_time}~(e); and for $2^2 \!\leq\! N \!\leq\! 2^{M}$, we plot the $|\phi_i (f, \mathcal{U}) - \phi_i (f, \mathcal{S})|$ versus the evaluation times $N$ in Figure~\ref{fig:run_time}~(f).
According to the experimental results, we have the following observations:
\begin{itemize}[leftmargin=10pt, topsep=0pt]
\setlength{\parskip}{1pt}
\setlength{\parsep}{0pt}
\setlength{\itemsep}{0pt}
    
    \item Approximately negative relationship between the absolute error $|\phi_i (f, \mathcal{U}) - \phi_i (f, \mathcal{S})|$ and the feature cross-contribution $\sum_{j \in \mathcal{S}} \eta_{i,j}$ can be observed in Figure~\ref{fig:run_time}~(e).
    
    
    \item According to Figure~\ref{fig:run_time}~(f), we have $|\phi_i (f, \mathcal{U})\! -\! \phi_i (f, \mathcal{S})|$ drops from maximum value to $0$ as $N$ grows, where more features are involved to be the cooperators.
    Meanwhile, \Algnameabbr{} achieves the lower bound of absolute error.
    
\end{itemize}
The above observations indicate the contributive cooperator selection of \Algnameabbr{} following Equation~(\ref{eq:S_selection2}) contributes to the minimization of the absolute estimation error, thus leading to accurate estimation of the feature contribution.


\subsection{Illustration of Explanation}
\label{sec:illusration_explanation}

We illustrate the Shapley explanation generated by \Algnameabbr{} for the DNN model prediction on the Census Income dataset to demonstrate the application of \Algnameabbr{}.
Specifically, we follow the settings in Appendix~\ref{sec:appendix_implementation_details} to train the DNN model $f$.
After the training, we select an instance $\tilde{\boldsymbol{x}} = [\tilde{x}_1, \cdots, \tilde{x}_M]$ from the testing dataset, and have $f_v(\mathcal{U}) = f^1(\tilde{\boldsymbol{x}}) - f^0(\tilde{\boldsymbol{x}}) = 1.73$, where $\text{softmax}[f(\boldsymbol{x})][i]$ denotes the predicted probability of the instance $\boldsymbol{x}$ belonging to class $i$. 
We have $i \in \{ 0,1 \}$ for the income prediction task on the Census Income dataset, where $i = 0$ or $1$ indicates an instance has income \emph{more than} or \emph{less than} 50K/yr, respectively.
The model outputs $\hat{y} = \text{sgn} [ f^1(\tilde{\boldsymbol{x}}) - f^0(\tilde{\boldsymbol{x}}) ] = 1$ which indicates $\tilde{\boldsymbol{x}}$ has the income \emph{more than} 50K/yr; and the base value of the model on the testing dataset satisfies $f_v(\varnothing) = f ( \bar{{\boldsymbol{x}}}_{\mathcal{U}} ) = -1.88$.
We illustrate the Shapley explanation generated by \Algnameabbr{} ($N=16$) in  Figure~\ref{fig:feature_contribution} of Appendix~\ref{sec:appendix_ex_illustration}, where the GT-Shapley value is also shown for comparison.
Overall, we have the following observations:
\begin{itemize}[leftmargin=10pt, topsep=0pt]
\setlength{\parskip}{1pt}
\setlength{\parsep}{0pt}
\setlength{\itemsep}{0pt}

    \item The summation of the $M$ feature contributions equals the distance between base value and model output, i.e. $\sum_{i=1}^M \hat{\phi}_i (f_v, \mathcal{U} ) = f_v(\mathcal{U}) - f_v(\varnothing) = 3.61$, where $M = 13$ for the Census Income dataset.
    
    \item The base value $f_v(\varnothing) < 0$ due to the fact that the classifier $f$ is learned based on unbalanced training dataset with more negative samples than positive samples.
    
    \item The top-three key features for the income prediction are \emph{education}, \emph{martial-status} and \emph{occupation}.
    
    \item The prediction of income has gender bias where \emph{gender=male} has positive contribution to the prediction.
    
\end{itemize}

\section{Conclusion}

In this work, we propose the Shapley chain rule and \Algnameabbr{} for the acceleration of Shapley explanation.
The Shapley chain rule provides the theoretical instructions to minimize the absolute estimation error, and \Algnameabbr{} follows the chain rule to accelerate the Shapley explanation via contributive cooperator selection.
The experimental results on three datasets using five evaluation metrics demonstrate that our proposed \Algnameabbr{} works more efficiently than state-of-the-art methods without degradation of interpretation performance, and indicate the potential application of \Algnameabbr{} to the real-world scenarios where the computational resource is limited.





\bibliography{reference-base} 
\bibliographystyle{icml2022}

\clearpage

\section*{Appendix}
\appendix
\setcounter{theorem}{0}

\section{Illustration of Explanation}
\label{sec:appendix_ex_illustration}

We illustrate the Shapley explanation generated by \Algnameabbr{} for the DNN model prediction on the Census Income dataset in this section.
Specifically, the details about training the DNN model can be referred to Appendix~\ref{sec:appendix_implementation_details}.
After training the DNN model $f$, we select an instance $\tilde{\boldsymbol{x}} = [\tilde{x}_1, \cdots, \tilde{x}_M]$ from the testing dataset, and have the value function given by $f_v(\mathcal{U}) = f^1(\tilde{\boldsymbol{x}}) - f^0(\tilde{\boldsymbol{x}}) = 1.73$, where $f^i(\boldsymbol{x})$ denotes the prediction of the instance $\boldsymbol{x}$ belonging to class~$i$.
For the selected instance $\tilde{\boldsymbol{x}}$, the DNN model outputs $\hat{y} = \text{sgn} [ f^1(\tilde{\boldsymbol{x}}) - f^0(\tilde{\boldsymbol{x}}) ] = 1$; and the base value of the model on the testing dataset satisfies $f_v(\varnothing) = f ( \bar{{\boldsymbol{x}}}_{\mathcal{U}} ) = -1.88$.
We illustrate the Shapley explanation generated by \Algnameabbr{} ($N=16$) in  Figure~\ref{fig:feature_contribution}, and give the GT-Shapley value for comparison.
We have several observations on the results which can be referred to Section~\ref{sec:illusration_explanation}.

\begin{figure}[h]
    \centering
    \setlength{\abovecaptionskip}{1mm}
    \setlength{\belowcaptionskip}{-3mm}
    \includegraphics[width=0.45\textwidth]{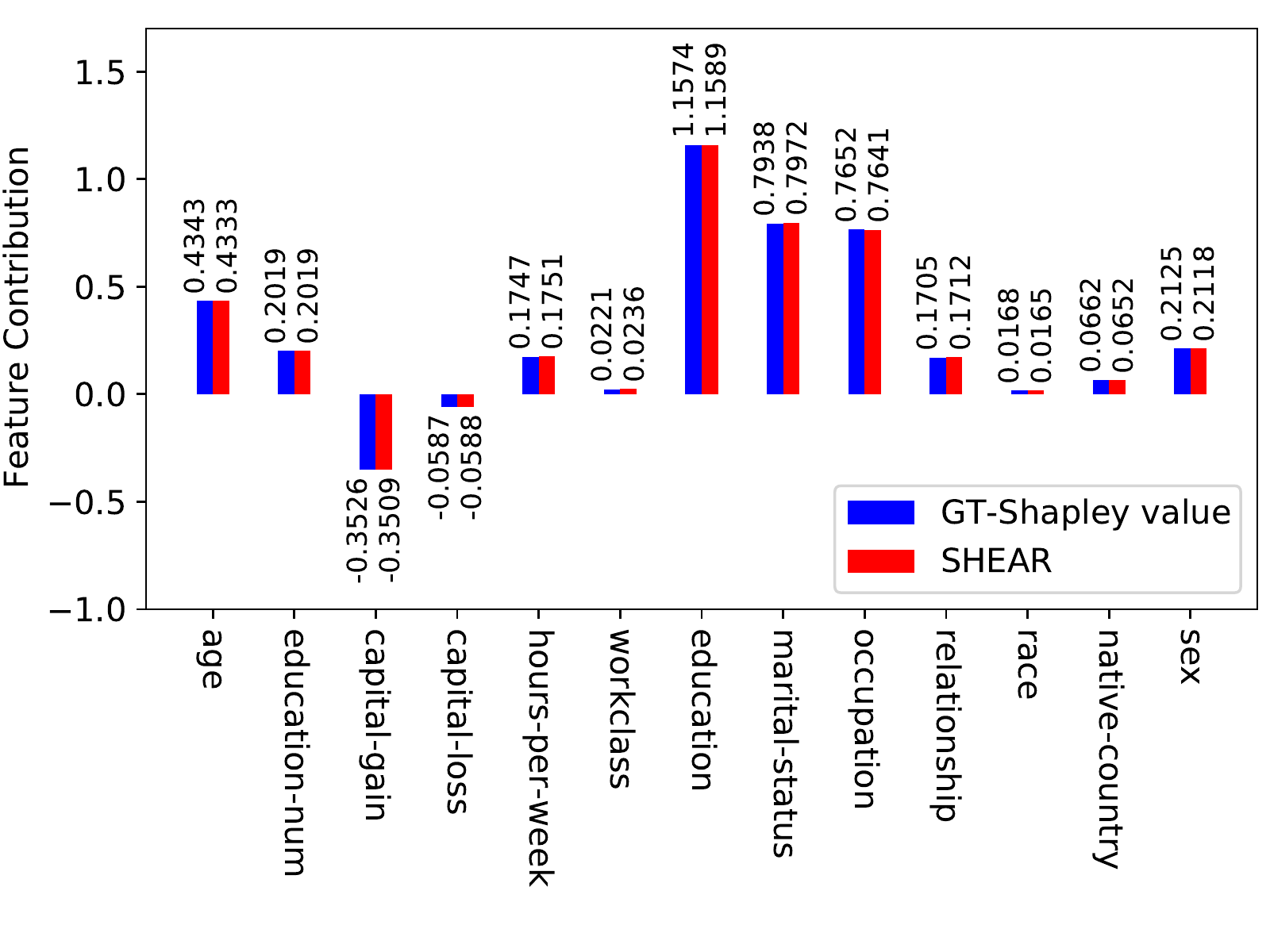}
    \caption{Explanation of model prediction for an instance from the Census Income dataset.}
    \label{fig:feature_contribution}
\end{figure}


\section{Proof of Theorem 1}
\label{sec:proof_theorem_1}

In this section, we first propose Corollaries 1 and 2, then utilize the corollaries to prove Theorem 1.


Before proving the theorem, we propose 
Corollaries~\ref{corol:expansion1} and~\ref{corol:expansion2}.

\begin{corollary}
\label{corol:expansion1}
Given function $f(x_i, \cdots): \mathcal{X} \to \mathbb{R}$ and $\bar{x}_i = \mathbb{E}_{x_i \sim p(x_i)}(x_i)$, $\forall (x_i, \cdots) \in \mathcal{X}$, $\exists$~function $g$ such that 
\begin{align}
f(x_i, \cdots) &= f(\bar{x}_i, \cdots) + g(x_i, \cdots) + o(x_i - \bar{x}_i),
\end{align}
\end{corollary}
where `$\cdots$' is the abbreviation of non-active variables.

\begin{proof}

Without loss of generality, we adopt Taylor's theorem to expand $f(x_i, \cdots)$ at point $(\bar{x}_i, \cdots)$ as follows,
\begin{align}
\label{eq:taylor_expansion1}
f(x_i, \cdots) &= f(\bar{x}_i, \cdots) \!+\! \frac{\partial f}{\partial X_i} (x_i \!-\! \bar{x}_i) \!+\! \frac{1}{2} \frac{\partial^2 f}{\partial X_i^2} (x_i \!-\! \bar{x}_i)^2 
\nonumber
\\
&+ o(x_i - \bar{x}_i).
\end{align}
Let $g(x_i, \cdots) = \frac{\partial f}{\partial X_i} (x_i - \bar{x}_i) + \frac{1}{2} \frac{\partial^2 f}{\partial X_i^2} (x_i \!-\! \bar{x}_i)^2$. 
Take the $g(x_i, \cdots)$ into Equation~(\ref{eq:taylor_expansion1}), we achieve $f(x_i, \cdots) = f(\bar{x}_i, \cdots) + g(x_i, \cdots) + o(x_i - \bar{x}_i)$.

\end{proof}

\begin{corollary}
\label{corol:expansion2}
Given function $f(x_i, x_j, \cdots): \mathcal{X} \to \mathbb{R}$ and $\bar{x}_i = \mathbb{E}_{x_i\sim p(x_i)}(x_i)$, $\forall \! (x_i, \! x_j, \! \cdots ) \!\in\! \mathcal{X}$, $\exists$~functions $g$ and $h \! : \! \mathcal{X} \!\!\to\! \mathbb{R}$ such that 
\begin{align}
&f(x_i, x_j, \cdots) 
\nonumber
\\
&= f(\bar{x}_i, \bar{x}_j, \cdots) + g(x_i, \bar{x}_j, \cdots) + h(\bar{x}_i, x_j, \cdots)
\\
&+ \lambda (x_i - \bar{x}_i) (x_j - \bar{x}_j) + o(x_i - \bar{x}_i) + o(x_j - \bar{x}_j),
\nonumber
\end{align}
where $\lambda = \frac{1}{2} \Big( \frac{\partial f}{\partial X_i \partial X_j} + \frac{\partial f}{\partial X_j \partial X_i} \Big) \Big|_{X_i=x_i, X_j=x_j, \cdots}$.
\end{corollary}

\begin{proof}

Without loss of generality, we employ Taylor's theorem to expand $f(x_i, x_j, \cdots)$ at point $(\bar{x}_i, \bar{x}_j, \cdots)$ towards variables $x_i$ and $x_j$ as follows,
\begin{align}
&f(x_i, x_j, \cdots) 
\nonumber
\\
&= f(\bar{x}_i, \bar{x}_j, \cdots) \!+\! \frac{\partial f}{\partial X_i} (x_i \!-\! \bar{x}_i) \!+\! \frac{1}{2} \frac{\partial^2 f}{\partial X_i^2} (x_i \!-\! \bar{x}_i)^2 
\nonumber
\\
\label{eq:taylor_expansion2}
&+ \frac{\partial f}{\partial X_j} (x_j - \bar{x}_j)
+ \frac{1}{2} \frac{\partial^2 f}{\partial X_j^2} (x_j - \bar{x}_j)^2 
\nonumber
\\
&+ \frac{1}{2} \Big( \frac{\partial^2 f}{\partial X_i \partial X_j} + \frac{\partial^2 f}{\partial X_j \partial X_i} \Big) (x_i - \bar{x}_i)(x_j - \bar{x}_j) 
\nonumber
\\
&+ o(x_i - \bar{x}_i) + o(x_j - \bar{x}_j).
\end{align}

Equation~(\ref{eq:taylor_expansion2}) can be reformulated into $f(x_i, x_j, \cdots) = f(x_i, x_j, \cdots) + g(x_i, \bar{x}_j, \cdots) + h(\bar{x}_i, x_j, \cdots) + o(x_i - \bar{x}_i) + o(x_j - \bar{x}_j)$, where
\begin{equation}
\begin{aligned}
g(x_i, \bar{x}_j, \cdots) &= \frac{\partial f}{\partial X_i} (x_i \!-\! \bar{x}_i) \!+\! \frac{1}{2} \frac{\partial^2 f}{\partial X_i^2} (x_i \!-\! \bar{x}_i)^2,
\\
h(\bar{x}_i, x_j, \cdots) &= \frac{\partial f}{\partial X_j} (x_j \!-\! \bar{x}_j) 
\!+\! \frac{1}{2} \frac{\partial^2 f}{\partial X_j^2} (x_j \!-\! \bar{x}_j)^2.
\end{aligned}
\end{equation}

\end{proof}


After
Corollaries~\ref{corol:expansion1} and~\ref{corol:expansion2}, we return to prove the theorem.

\begin{theorem}[Shapley Chain Rule]

\label{theorem:shapley_chain_rule}
For any differentiable value function $f_v: 2^M \to \mathbb{R}$, the contribution of feature $i$ to $f_v(\mathcal{U})$ satisfies 
\begin{equation}
\phi_i( f_v, \mathcal{U} ) = \phi_i( f_v, \mathcal{U} \setminus \{j\} ) + \Delta_{i,j} + o_{i,j}, 
\nonumber
\end{equation}
where $j \in \mathcal{U} \setminus \{i\}$ denotes another feature; $\phi_i( f_v, \mathcal{U} \setminus \{j\} )$ denotes the contribution of feature $i$ to $f_v(\mathcal{U} \setminus \{j\})$; $o_{i,j} \!=\! o(x_i \!-\! \bar{x}_i) \!+\! o(x_j \!-\! \bar{x}_j)$; and the error term $\Delta_{i,j}$ is given by
\begin{align} 
\Delta_{i,j} = &(x_i - \bar{x}_i) (x_j - \bar{x}_j) 
\nonumber
\\
\label{eq:appendx_errorterm2}
&\!\!\!\!\! \sum_{\boldsymbol{S} \subseteq \mathcal{U} \setminus \{ i,j \}} \!\!\!\!\!\!\!\!\! \frac{ \nabla_{i,j}^2 f_v( \boldsymbol{S} \!\cup\! \{ i,j \}) \!+\! \nabla_{j,i}^2 f_v( \boldsymbol{S} \!\cup\! \{ i,j \} ) }{2 (M - |\boldsymbol{S}| - 1) \binom{M}{|\boldsymbol{S}|+1}},
\end{align} 
where $x_i$ denotes the value of feature $i$; $\bar{x}_i$ denotes the reference value of feature $i$; and $\nabla_{i,j}^2 f_v ( \boldsymbol{S} ) = \frac{\partial^2 f ( \boldsymbol{x}_{\boldsymbol{S}},  \bar{\boldsymbol{x}}_{\mathcal{U} \setminus \boldsymbol{S}} )}{\partial x_i \partial x_j}$ denotes the cross-gradient towards $x_i$ and $x_j$. 
\end{theorem}

\begin{proof}

For the simplicity of deduction but without loss of generality, we regard features $i$ and $j$ as the active variables and adopt the abbreviation $f_v(\{ i \} \cup \boldsymbol{S}) = f(x_i, \bar{x}_j, \cdots)$ and $f_v(\boldsymbol{S}) = f(\bar{x}_i, \bar{x}_j, \cdots)$ for $\boldsymbol{S} \subseteq \mathcal{U} \setminus \{ i,j \}$, where the non-active features are abbreviated into~`$\cdots$'.
Taking Corollary~\ref{corol:expansion1} into $f_v$, we have
\begin{align}
&f_v(\{ i \} \cup \boldsymbol{S}) - f_v(\boldsymbol{S}) 
\nonumber
\\
&= f_v(x_i, \bar{x}_j, \cdots) - f_v(\bar{x}_i, \bar{x}_j, \cdots)
\nonumber
\\
\label{eq:taylor_expansion1_2}
&= g(x_i, \bar{x}_j, \cdots) + o(x_i - \bar{x}_i).
\end{align}
Taking Equation~(\ref{eq:taylor_expansion1_2}) into Equation~(\ref{eq:shapley_value}), we have the feature contribution $\phi_i (f, \mathcal{U} \setminus \{ j \})$ given by
\begin{align}
&\phi_i (f, \mathcal{U} \setminus \{ j \}) 
\nonumber
\\
&= \frac{1}{M-1} \!\!\!\! \sum_{\boldsymbol{S} \subseteq \mathcal{U} \setminus \{ i,j \}} \!\!\!\! \binom{M-2}{|\boldsymbol{S}|}^{-1} \big[ f_v(\{ i \} \cup \boldsymbol{S}) \!-\! f_v(\boldsymbol{S}) \big]
\nonumber
\\
\label{eq:phi_f_U__j}
&= \frac{1}{M-1} \!\!\!\!\!\!\! \sum_{\boldsymbol{S} \subseteq \mathcal{U} \setminus \{ i,j \}} \!\!\!\!\!\!\!\! \binom{M-2}{|\boldsymbol{S}|}^{-1} \!\!\!\!\!\! g(x_i, \bar{x}_j, \cdots) \!+\! o(x_i \!-\! \bar{x}_i). \!\!\!\!
\end{align}



Similarly, we have the abbreviation $f_v(\{ i,j \} \cup \boldsymbol{S}) = f(x_i, x_j, \cdots)$ and $f_v(\{ j \} \cup \boldsymbol{S}) = f(\bar{x}_i, x_j, \cdots)$ for $\boldsymbol{S} \subseteq \mathcal{U} \setminus \{ i \}$. 
Taking Corollary~\ref{corol:expansion2} into $f_v$, we have
\begin{align}
\label{eq:taylor_expansion2_2}
&f_v(\{ i,j \} \cup \boldsymbol{S}) - f_v(\{ j \} \cup \boldsymbol{S}) 
\nonumber
\\
&= f(x_i, x_j, \cdots) - f(\bar{x}_i, x_j, \cdots)
\\
&= g(x_i, \bar{x}_j, \cdots) + \lambda_{i,j} (x_i - \bar{x}_i) (x_j - \bar{x}_j) 
\nonumber
\\
&+  o(x_i - \bar{x}_i) + o(x_j - \bar{x}_j),
\nonumber
\end{align}
where $\lambda_{i,j} = \frac{1}{2} \Big( \frac{\partial f}{\partial X_i \partial X_j} + \frac{\partial f}{\partial X_j \partial X_i} \Big) \Big|_{X_i=x_i, X_j=x_j, \cdots}$.

According to Equation~(\ref{eq:shapley_value}), the feature contribution $\phi_i (f, \mathcal{U})$ can be reformulated as follows,
\begin{align}
&\phi_i (f, \mathcal{U}) = \frac{1}{M} \!\!\!\! \sum_{\boldsymbol{S} \subseteq \mathcal{U} \setminus \{ i \}} \!\!\!\! \binom{M-1}{|\boldsymbol{S}|}^{-1} \!\!\!\!\!\! \big[ f_v(\{ i \} \cup \boldsymbol{S}) \!-\! f_v(\boldsymbol{S}) \big]
\nonumber
\\
\label{eq:phi_f_U4}
&= \frac{1}{M} \sum_{\boldsymbol{S} \subseteq \mathcal{U} \setminus \{ i,j \}} \!\! \binom{M-1}{|\boldsymbol{S}|}^{-1} \!\!\!\!\!\! \big[ f_v(\{ i \} \cup \boldsymbol{S}) - f_v(\boldsymbol{S}) \big] 
\\
\label{eq:phi_f_U5}
&+\! \frac{1}{M}  \!\!\!\!\! \sum_{\boldsymbol{S} \subseteq \mathcal{U} \setminus \{ i,j \}}  \!\!\!\!\!  \binom{M \!-\! 1}{|\boldsymbol{S}| \!+\! 1}^{-1} \!\!\!\!\!\!\! \big[ f_v(\{ i,j \} \!\cup\! \boldsymbol{S}) \!-\! f_v(\{ j \} \!\cup\! \boldsymbol{S}) \big] \!. \!\!\!
\end{align}
Taking Equation~(\ref{eq:taylor_expansion1_2}) into Equation~(\ref{eq:phi_f_U4}) and taking Equation~(\ref{eq:taylor_expansion2_2}) into Equation~(\ref{eq:phi_f_U5}), we have
\begin{align}
\label{eq:phi_f_U1}
&\phi_i (f, \mathcal{U}) 
\nonumber
\\
&= \! \frac{1}{M} \!\!\!\!\!\!\! \sum_{\boldsymbol{S} \subseteq \mathcal{U} \setminus \{ i,j \}} \!\!\! \Bigg[ \!\! \binom{M\!-\!1}{|\boldsymbol{S}|}^{-1} \!\!\!\!\!\!\! + \! \binom{M\!-\!1}{|\boldsymbol{S}|\!+\!1}^{-1} \! \Bigg] \! g(x_i, \! \bar{x}_j, \! \cdots) \!\!\!\!\!\!\!\!\!\!\!\!\!\!\!\!\!\!
\\
\label{eq:phi_f_U2}
&+ \frac{1}{M} \!\!\! \sum_{\boldsymbol{S} \subseteq \mathcal{U} \setminus \{ i,j \}} \!\!\! \binom{M-1}{|\boldsymbol{S}|+1}^{-1} \!\!\!\! \lambda_{i,j} (x_i - \bar{x}_i) (x_j - \bar{x}_j) 
\\
\label{eq:phi_f_U3}
&+ o(x_i - \bar{x}_i) + o(x_j - \bar{x}_j).
\end{align}
Theorem 1 can be proved via taking Equation~(\ref{eq:phi_f_U__j}) into Equation~(\ref{eq:phi_f_U1}); transforming Equation~(\ref{eq:phi_f_U2}) into Equation~(\ref{eq:appendx_errorterm2}); and having $o_{i,j} = o(x_i - \bar{x}_i) + o(x_j - \bar{x}_j)$.

\end{proof}





\section{Proof of Theorem 2}
\label{sec:proof_theorem_2}

We prove Theorem 2 in this section.

\begin{theorem}[Upper Bound of Error Term]
For any features $i \neq j \in \mathcal{U}$, the upper bound of the absolute gap between $\phi_i( f_v, \mathcal{U} )$ and $\phi_i( f_v, \mathcal{U} \setminus \{ j \} )$ is given by
\begin{equation}
\setlength\abovedisplayskip{2mm}
\setlength\belowdisplayskip{2mm}
\big| \phi_i( f_v, \mathcal{U} ) \!-\! \phi_i( f_v, \mathcal{U} \setminus \{ j \} )  \big| \!\leq\! \epsilon_{i,j} |x_i \!-\! \bar{x}_i| |x_j \!-\! \bar{x}_j|,
\end{equation}
where $\epsilon_{i,j}$ relies on the gradient towards $x_i$ and $x_j$ by
\begin{equation}
\setlength\abovedisplayskip{1mm}
\setlength\belowdisplayskip{2mm}
\label{eq:appendix_epsilon2}
\epsilon_{i,j} \!=\! \max\limits_{\mathcal{V} \subseteq \mathcal{U} \setminus \{ i,j \}} \frac{1}{4} \big| \nabla_{i,j}^2 f_v (\mathcal{U} \setminus \mathcal{V}) \!+\! \nabla_{j,i}^2 f_v (\mathcal{U} \setminus \mathcal{V}) \big|.
\end{equation}
\end{theorem}

\begin{proof}

We ignore the infinitesimal $o_{i,j}$ in Theorem 1 to prove the upper bound.
According to Theorem 1, we have
\begin{align}
\label{eq:proof_theorem2_1}
\!\!\!\! \big| \phi_i( f_v, \mathcal{U} ) \!-\! \phi_i( f_v, \mathcal{U} \setminus \{ j \} )  \big| \!\leq\! w_{i,j} |x_i \!-\! \bar{x}_i| |x_j \!-\! \bar{x}_j|, \!\!\!\!
\end{align}
where $w_{i,j}$ is given by
\begin{align}
w_{i,j} &= \sum_{\boldsymbol{S} \subseteq \mathcal{U} \setminus \{ i,j \}} \!\!\!\!\! \frac{ \nabla_{i,j}^2 f_v( \boldsymbol{S} \!\cup\! \{ i,j \}) \!+\! \nabla_{j,i}^2 f_v( \boldsymbol{S} \!\cup\! \{ i,j \} ) }{2 (M - |\boldsymbol{S}| - 1) \binom{M}{|\boldsymbol{S}|+1}}.
\nonumber
\end{align}
Define $\epsilon_{i,j}$ following Equation~(\ref{eq:appendix_epsilon2}), we have
\begin{align}
\label{eq:proof_theorem2_2}
w_{i,j} &\leq \epsilon_{i,j} \!\!\!\!\! \sum_{\boldsymbol{S} \subseteq \mathcal{U} \setminus \{ i,j \}} \frac{ 2 }{ (M - |\boldsymbol{S}| - 1) \binom{M}{|\boldsymbol{S}|+1}} = \epsilon_{i,j}.
\end{align}
Eventually, Theorem 2 can be proved via taking Equation~(\ref{eq:proof_theorem2_2}) to Equation~(\ref{eq:proof_theorem2_1}), where 
\begin{equation}
\begin{aligned}
&\sum_{\boldsymbol{S} \subseteq \mathcal{U} \setminus \{ i,j \}} \frac{ 2 }{ (M - |\boldsymbol{S}| - 1) \binom{M}{|\boldsymbol{S}|+1}}
\\
&= 2\sum_{|\boldsymbol{S}| = 0}^{M-2} \frac{ \binom{M-2}{|\boldsymbol{S}|} }{ (M - |\boldsymbol{S}| - 1) \binom{M}{|\boldsymbol{S}|+1}}
\\
&= 2\sum_{|\boldsymbol{S}| = 0}^{M-2} \frac{(M-|\boldsymbol{S}|-1)!(|\boldsymbol{S}|+1)!(M-2)!}{(M-|\boldsymbol{S}|-1)M!(M-|\boldsymbol{S}|-2)!|\boldsymbol{S}|!}
\\
&= \frac{2\sum_{|\boldsymbol{S}| = 0}^{M-2}(|\boldsymbol{S}|+1)}{M(M-1)} = 1.
\nonumber
\end{aligned}
\end{equation}

\end{proof}

\section{Proof of Remark 2}
\label{sec:proof_remark_2}

We prove Remark 2 in this section.

\begin{proof}
Remark 2 generalizes Theorem 2 to multiple features.
Specifically, we consider a subset of features $\mathcal{V} \!=\! \{j_1, \! j_2, \! \cdots \! \}$.
According to Equation~(\ref{eq:error_upper_bound}), we have
\begin{equation}
\begin{aligned}
& \big| \phi_i(f_v, \mathcal{U}) - \phi_i(f_v, \mathcal{U} \setminus \mathcal{V} ) \big|
\\
&= | \phi_i(f_v, \! \mathcal{U}) \!-\! \phi_i(f_v, \! \mathcal{U} \!\setminus\! \{j_1\} ) \\ &\quad\quad\quad\quad \!+\! \phi_i(f_v, \! \mathcal{U} \setminus \{j_1\} ) -\! \phi_i(f_v, \! \mathcal{U} \!\setminus\! \{j_1, j_2 \}) \\ &\quad\quad\quad\quad\quad\quad\quad\quad \!+\! \cdots |
\\
&\leq | \phi_i( f_v, \! \mathcal{U} ) \!-\! \phi_i( f_v, \! \mathcal{U} \!\setminus\! \{j_1\})| \\ &\quad\quad\quad\quad  \!+\! | \phi_i( f_v, \! \mathcal{U} \!\setminus\! \{j_1\} ) \!-\! \phi_i( f_v, \! \mathcal{U} \!\setminus\! \{j_1, j_2\} ) | \\ &\quad\quad\quad\quad\quad\quad\quad\quad \!+\! \cdots
\\
&= \sum_{j \in \mathcal{V}} \epsilon_{i,j} |x_i \!-\! \bar{x}_i| |x_j \!-\! \bar{x}_j|
\nonumber
\end{aligned}
\end{equation}
After taking $\mathcal{S} = \mathcal{U} \!\setminus\! \{i\} \!\setminus\! \mathcal{V}$, we have 
\begin{align}
\big| \phi_i( f_v, \mathcal{U} ) \!-\! \phi_i( f_v, \mathcal{S} \!\cup\! \{ i \} ) \big| \!\leq\!\!\!\!\!\!\! \sum_{j \in \mathcal{U} \setminus \{ i \} \setminus \mathcal{S} } \!\!\!\!\!\! \epsilon_{i,j} |x_i \!-\! \bar{x}_i| |x_j \!-\! \bar{x}_j|
\nonumber
\end{align}
\end{proof}

\section{Proof of Equation~(\ref{eq:S_selection_0}) to Equation~(\ref{eq:S_selection})}

We give the detailed proof from  Equation~(\ref{eq:S_selection_0}) to Equation~(\ref{eq:S_selection}) in this section.

\begin{proof}
According to Equation~(\ref{eq:S_selection}), for $\mathcal{S} \subseteq \mathcal{U} \setminus \{ i \}$, we have that
\begin{align}
&\sum_{j \in \mathcal{U} \setminus \mathcal{S}} \epsilon_{i,j} |x_i - \bar{x}_i| |x_j - \bar{x}_j| 
\\
&= \sum_{j \in \mathcal{U}} \epsilon_{i,j} |x_i - \bar{x}_i| |x_j - \bar{x}_j| 
\\&- \sum_{j\in \mathcal{S}} \epsilon_{i,j} |x_i - \bar{x}_i| |x_j - \bar{x}_j|,
\end{align}
where $\sum_{j \in \mathcal{U}}  \epsilon_{i,j} |x_i - \bar{x}_i| |x_j - \bar{x}_j|$ is a constant for the target feature~$i$.
In this way, Equation~(\ref{eq:S_selection}) can be considered as the dual problem of Equation~(\ref{eq:S_selection_0}).
\end{proof}

\section{Details about the Datasets}
\label{sec:appendix_dataset}

\begin{table*}[]
\centering
\caption{Dataset Statistics}
\begin{tabular}{l|c|c|c|c|c}
\hline
Dataset &  Continuous  &  Categorical  &  Training  &  Validation  &  Testing  \\
\hline
{Census Income} & 5 & 8 & 20838 & 5210 & 6513 \\
\hline
{German Credit} & 7 & 9 & 28934 & 7234 & 9043 \\
\hline
Cretio & 13 & 26 & 80000 & 10000 & 10000 \\
\hline
\end{tabular}
\label{tab:dataset}
\end{table*}

We give the details about the datasets in this section, and the dataset statistics are shown in Table~\ref{tab:dataset}.
\begin{itemize}[leftmargin=10pt, topsep=0pt]
\setlength{\parskip}{1pt}
\setlength{\parsep}{0pt}
\setlength{\itemsep}{0pt}

\item \textbf{Census Income}\footnote{\url{https://archive.ics.uci.edu/ml/datasets/census+income}}:
In this dataset, each sample has five continuous features and eight one-hot encoded categorical features.
The task for this dataset is to predict whether a person has income \emph{more than}~(\emph{target}=1) or \emph{less than}~(\emph{target}=0) 50K/yr based on her/his personal features~(\emph{education, occupation, working hours, etc.}).

\item \textbf{German credit}\footnote{ \url{https://archive.ics.uci.edu/ml/datasets/statlog+(german+credit+data)}}: 
The samples in this dataset have seven continuous features and nine one-hot encoded categorical features.
The task is to predict whether a person has \emph{good}~(\emph{y}=1) or \emph{bad}~(\emph{y}=0) credit risks based on her/his personal features~(\emph{job, education, balance, etc.}).

\item \textbf{Criteo}\footnote{\url{https://www.kaggle.com/c/criteo-display-ad-challenge/data}}\footnote{\url{https://labs.criteo.com/2014/02/kaggle-display-advertising-challenge-dataset/}}: 
Each sample in this dataset corresponds to a user and an ad which have 13 continuous features and 26 one-hot encoded categorical features.
The Criteo dataset is widely used in recommender systems, where the task is to predict the clicking rate of a user on an ad.


\end{itemize}

\section{Details about the Baseline Methods}
\label{sec:appendix_baseline}

We give the details about the baseline methods in this section.

\begin{itemize}[leftmargin=10pt, topsep=0pt]
\setlength{\parskip}{1pt}
\setlength{\parsep}{0pt}
\setlength{\itemsep}{0pt}

    \item \textbf{Kernel-SHAP}~\cite{lundberg2017unified}: 
    Kernel-SHAP is a model agnostic method for feature importance estimation, which uses specific linear regressions to approach the original model outputs at each data point so that the linear weights approach the feature contribution.
    Kernel-SHAP estimates the feature contributions by
    \begin{align}
    \label{eq:KS_feature_contribution}
    [\hat{\phi}_1, \cdots, \hat{\phi}_M] = \big( \mathbf{A}^{\mathsf{T}} \mathbf{W} \mathbf{A} \big)^{-1} \mathbf{A}^{\mathsf{T}} \mathbf{W} \mathbf{b},
    \end{align}
    where $\mathbf{A} = [\mathbf{1}_{\boldsymbol{S}_{1}}^{\mathsf{T}}, \cdots, \mathbf{1}_{\boldsymbol{S}_{N}}^{\mathsf{T}}]$; $\mathbf{W} = \text{diag} \{ \pi_1, \cdots, \pi_N \}$; and for $1 \leq n \leq N$, $\pi_n$ is given by
    \begin{align}
        \pi_n = \frac{M-1}{\binom{M}{|\boldsymbol{S}_n|}|\boldsymbol{S}_n||M-|\boldsymbol{S}_n||}.
    \end{align}

    \item \textbf{KS-WF}~\cite{covert2021improving}: 
    KS-WF adopts the Welford algorithm~\cite{welford1962note} to calculate and reduce the variance of the feature contribution estimation. 
    Specifically, in each iteration $n$, it randomly selects feature subset $\boldsymbol{S}_{n}$, and calculates the coefficient matrix by
    \begin{align}
        \mathbf{A}_{n} &= \Big( 1-\frac{1}{n} \Big) \mathbf{A}_{n-1} + \frac{1}{n} \mathbf{1}_{\boldsymbol{S}_{n}}\mathbf{1}_{\boldsymbol{S}_{n}}^{\mathsf{T}}
        \nonumber
        \\
        \mathbf{b}_{n} &= \Big( 1-\frac{1}{n} \Big) \mathbf{b}_{n-1} + \frac{1}{n}  \Big( f_v(\mathcal{U}) - f_v(\mathcal{\varnothing}) \Big) \mathbf{1}_{\boldsymbol{S}_{n}},
        \nonumber
    \end{align}
    where $\mathbf{A}_0 = \mathbf{0}_{M \times M}$ and $\boldsymbol{b}_0 = \mathbf{0}_{M}$ for the first iteration.
    After the $n$ iterations, the feature contribution are estimated by
    \begin{align}
    [\hat{\phi}_1, \! \cdots \! , \! \hat{\phi}_M] \!=\! \mathbf{A}_{n}^{-1} \! \Big( \mathbf{b}_n \!-\! \mathbf{1}\frac{\mathbf{1}^{\mathsf{T}}  \mathbf{A}_{n}^{-1} \mathbf{b} \!-\! f_v(\mathcal{U}) \!+\! f_v(\mathcal{\varnothing})}{\mathbf{1}^{\mathsf{T}}  \mathbf{A}_{n}^{-1} \mathbf{1}} \Big).
    \nonumber
    \end{align}
    The iterative process of KS-WF has to follow the feed-forward pipeline without parallelization which constrains its computational efficiency.

    \item \textbf{KS-Pair}~\cite{covert2021improving}:
    KS-Pair adopts the pair sampling to accelerate the convergence of Kernel-SHAP.
    Specifically, KS-Pair samples pairs of input feature coalitions $(\boldsymbol{S}_n, \mathcal{U} \setminus \boldsymbol{S}_n)|_{1 \leq n \leq \frac{N}{2}}$ and estimates the feature contribution following Equation~(\ref{eq:KS_feature_contribution}).
    
    \item \textbf{PS}~\cite{mitchell2021sampling}: 
    The permutation sampling method estimates the feature contribution merely based on model inference.
    It randomly masks each feature for each data point and takes the average variety of model output as the feature importance.
    Given the model value function $f_v$ and input feature coalition $\boldsymbol{S}_1, \cdots, \boldsymbol{S}_N$, the feature contribution is estimated given by
    \begin{align}
    \label{eq:PS_feature_contribution}
        \hat{\phi}_i = \frac{1}{N} \sum_{j=1}^N f_v(\boldsymbol{S}_n \cup \{ i \}) - f_v(\boldsymbol{S}_n).
    \end{align}

    \item \textbf{APS}~\cite{lomeli2019antithetic, mitchell2021sampling}: 
    APS adopts the antithetical sampling to reduce the variance of feature contribution estimation.
    To be concrete, half of the feature coalitions are randomly selected from the feature space, and the remaining feature coalitions takes ${S}_{N+1-n} = \mathcal{U} \setminus \boldsymbol{S}_n$ for $\frac{N}{2} < n \leq N$.
    The feature contribution is estimated following Equation~(\ref{eq:PS_feature_contribution}).
    
\end{itemize}

\section{Implementation Details}

\label{sec:appendix_implementation_details}

The experiment on each dataset follows the pipeline of \emph{model training}: training the DNN model; \emph{interpretation benchmark}: adopting the brute-force algorithm to calculate the exact Shapley value as the ground truth explanation for the evaluation; and \emph{interpretation evaluation}: evaluating the interpretation methods.
Each step is specified as follows.

\noindent
\textbf{Model Training}:
We adopt 3-layer MLP (multi-layer perceptron) as the classification model for the Census Income and German Credit datasets.
To train the model, we adopt Adam optimizer with $10^{-3}$ learning rate to update the model parameters so that the cross-entropy loss can be minimized.
Adopting early stopping based on the performance on the validation dataset to avoid overfitting, we achieve $84.7\%$ and $89.8\%$ accuracy on the Census Income and German Credit testing set, respectively.
For the Cretio dataset, we use DeepFM~\cite{guo2017deepfm} as the model and adopt Adam optimizer with $10^{-4}$ learning rate to update the parameters.
Other settings are the same with that of Census Income dataset, and we achieve $71.09\%$ accuracy on the testing dataset.

\noindent
\textbf{Interpretation Benchmark}: We adopt the brute-force algorithm to calculate the ground-truth Shapley value (GT-Shapley value) for the evaluation.
Specifically, the GT-Shapley values are calculated according to Equation~(\ref{eq:shapley_value}), where $f_v(\boldsymbol{S})$ is given by Equation~(\ref{eq:value_func_approx}); the reference value of continuous features takes the mean value for all datasets; and that of categorical features takes the mean value for the Census Income and German Credit dataset, and takes the mode for the Cretio dataset\footnote{DeepFM works based on the Hash-index of categorical features. Instead of the mean value of Hash-index, we adopt the mode for the reference value of categorical features.}.
Other hyper-parameter settings are summarized in Table~\ref{tab:hyperparam}.  

\begin{table}[]
\caption{Hyper-parameter setting for model training and Shapley value benchmark.}
\centering
\begin{tabular}{l|c|c|c}
\hline
Dataset & {\small Census Income} & {\small German Credit} & Cretio \\
\hline
Model & {\small 3-layer MLP} & {\small 3-layer MLP} & {\small DeepFM} \\
\hline
{\small Hidden dim.} & 64 & 64 & 32 \\
\hline
Opt., LR & Adam,{\scriptsize $10^{-3}$} & Adam,{\scriptsize $10^{-3}$} & $\!\!$ Adam,{\scriptsize $10^{-4}$} $\!\!\!\!\!\!$ \\
\hline
Batch Size & 256 & 256 & 256 \\
\hline
{$\text{Ref}_{\text{continuous}}$} & Mean & Mean & Mean \\
\hline
{$\text{Ref}_{\text{categorical}}$} & Mean & Mean & Mode \\
\hline
\end{tabular}
\label{tab:hyperparam}
\end{table}

\noindent
\textbf{Interpretation Evaluation}:
\Algnameabbr{} and baseline methods are employed to generate interpretations for the instances in the testing set.
To evaluate the interpretation generated by different methods, we have five evaluation metrics given in Sections \ref{sec:GT_eval_metric}, \ref{sec:NGT_eval_metric} and \ref{sec:exp_throughput}, including two metrics taking the GT-Shapley value as the reference: the absolute estimation error and accuracy of feature importance ranking; two metrics evaluating the interpretation via model perturbation: Faithfulness and Monotonicity; and the algorithmic throughput to evaluate the running speed of generating the explanation.
The evaluation metrics are calculated for each testing instance, and we report the mean and standard deviation of the metric value to illustrate the interpretation performance and variance of \Algnameabbr{} and baseline methods.



In the evaluation step, we have to unify the experimental conditions over different interpretation mechanisms to achieve a fair comparison. 
In particular, \Algnameabbr{} and Permutation-based methods (i.e., PS and APS) take $NM$ times of model evaluation to estimate the contributions of $M$ features, while Kernel-SHAP-based methods (i.e., Kernel-SHAP, KS-WF and KS-Pair) simply take $N$ times with the merit of matrix operations.
For unified the conditions, we should make sure Kernel-SHAP-based methods can also benefit from $NM$ times of model evaluation.
Note that setting $NM$ times of model evaluation in a single batch would not work for Kernel-SHAP-based methods in practice due to the $M^2$-time-related memory cost and computation load.
We choose to execute Kernel-SHAP-based methods by $M$ loops in our experiments, and average the contribution values as the output. 
In this way, we control the model evaluations, overall memory cost and computation load equal for all involved methods.





\section{Details about Computing Infrastructure}
\label{sec:appendix_infrastructure}

The details about our physical computing infrastructure for testing the algorithmic throughput are given in Table~\ref{tab:computing_infrastructure}.
\begin{table}[H]
\centering
\caption{Computing infrastructure for the experiments.}
\begin{tabular}{l|c}
\hline
Device Attribute & Value \\
\hline
Computing infrastructure & CPU \\
CPU model & Apple M1 \\
CPU number & 1 \\
Core number & 8 \\
Memory size & 16GB \\
\hline
\end{tabular}
\label{tab:computing_infrastructure}
\vspace{-3mm}
\end{table}
We have not used GPUs in our experiments because \Algnameabbr{} and baseline methods mostly rely on DNN evaluation instead of the training of DNNs.
Since the DNN model for the interpretation is controlled equal for all methods, the ranking of algorithmic throughput in Figures~\ref{fig:run_time}~(a)-(c) should be roughly consistent with the testing results on other types of equipment.

\end{document}